\documentclass{article}

\usepackage[main, final]{neurips_2026}

\usepackage[utf8]{inputenc} 
\usepackage[T1]{fontenc}    
\usepackage{hyperref}       
\usepackage{url}            
\usepackage{booktabs}       
\usepackage{amsfonts}       
\usepackage{nicefrac}       
\usepackage{microtype}      
\usepackage{xcolor}         

\usepackage{mathtools, dsfont, amsthm, array, xfrac, upquote}
\usepackage{enumitem}
\usepackage{wrapfig, caption, multirow, makecell}
\usepackage{algorithm, algpseudocode}
\usepackage{graphicx, subcaption}
\usepackage{listings}
\newtheorem{defn}{Definition}
\newtheorem{thm}{Theorem}
\newtheorem*{thm_}{Theorem}
\newtheorem{prop}{Proposition}
\newtheorem*{prop_}{Proposition}

\newtheorem{lem}{Lemma}
\newcommand{\grpo}{\textsc{GRPO}}
\newcommand{\name}{\textsc{POETS}}
\newcommand{\tosfit}{\textsc{ToSFiT}}

\newcommand{\esllm}{\textsc{EvoSearch (LLM)}}

\newcommand{\es}{\textsc{EvoSearch}}
\newcommand{\fibo}{\textsc{FIBO}}
\definecolor{Green}{RGB}{0, 100, 0}
\definecolor{Red}{RGB}{200, 0, 0}
\definecolor{rebuttal}{RGB}{0, 0, 255}

\usepackage[most]{tcolorbox}
\newtcolorbox[auto counter]{takeaway}[1][]{%
  enhanced,
  breakable,
  colback=black!5,          
  colframe=black!85,        
  boxrule=1.25pt,
  arc=3pt,                  
  left=2mm,right=2mm,top=2mm,bottom=1.3mm,
  before skip=10pt, after skip=10pt,
  title={Takeaway~\thetcbcounter},
  colbacktitle=black!85,    
  coltitle=white,           
  fonttitle=\bfseries\small,
  attach boxed title to top left={yshift=-1.2mm, xshift=2mm},
  boxed title style={
    enhanced,
    arc=3pt,
    top=0.5mm, bottom=0.5mm, left=1mm, right=1mm,
    boxrule=0pt,           
    interior engine=empty, 
  },
  #1                        
}

\usepackage{booktabs,arydshln}

\makeatletter
\def\adl@drawiv#1#2#3{%
        \hskip.5\tabcolsep
        \xleaders#3{#2.5\@tempdimb #1{1}#2.5\@tempdimb}%
                #2\z@ plus1fil minus1fil\relax
        \hskip.5\tabcolsep}
\newcommand{\cdashlinelr}[1]{%
  \noalign{\vskip\aboverulesep
           \global\let\@dashdrawstore\adl@draw
           \global\let\adl@draw\adl@drawiv}
  \cdashline{#1}
  \noalign{\global\let\adl@draw\@dashdrawstore
           \vskip\belowrulesep}}
\makeatother

\title{POETS: Uncertainty-Aware LLM Optimization\\ via Compute-Efficient Policy Ensembles}

\author{
Nicolas Menet$^{1,2}$\\
{\tt\small nicolas.menet@ibm.com}\\
\And
Andreas Krause$^{2}$\\
{\tt\small krausea@ethz.ch}\\
\And
Abbas Rahimi$^{1}$\\
{\tt\small abr@zurich.ibm.com}
\And
{
\normalfont $^{1}$IBM Research -- Zurich, $^{2}$Department of Computer Science, ETH Zürich}
}

\begin{document}

\maketitle

\begin{abstract}
  Balancing exploration with exploitation is a core challenge in black-box optimization and reinforcement learning. We introduce \name{} (\textbf{Po}licy \textbf{E}nsembles for \textbf{T}hompson \textbf{S}ampling), a novel framework that bridges uncertainty quantification and policy optimization. Our approach is grounded in the insight that policies trained with Kullback-Leibler (KL) regularization implicitly encode an underlying reward function. Building on this, \name{} bypasses the complex, nested process of training an uncertainty-aware reward model and separately fitting a policy to this model. Instead, we directly train a policy ensemble to capture epistemic uncertainty by matching implicitly encoded reward functions to online, bootstrapped data. To overcome the prohibitive compute and memory constraints of ensembling Large Language Models (LLMs), \name{} utilizes an efficient architecture: the ensemble shares a pre-trained backbone while maintaining diversity through independent Low-Rank Adaptation (LoRA) branches. Theoretically, we show that \name{} implicitly approximates KL-regularized Thompson sampling, a principled optimization strategy with optimal cumulative regret bounds of ${\mathcal O}(\sqrt{T \gamma_T})$. Empirically, we demonstrate that \name{} achieves state-of-the-art sample efficiency across diverse scientific discovery domains, including protein search, therapeutic molecular design, and quantum circuit design. \name{} also improves the optimization trajectory and out-of-distribution transfer of reinforcement learning across mathematical reasoning and code generation tasks, proving particularly robust in off-policy settings with experience replay or in small dataset regimes.\looseness -1
\end{abstract}

\section{Introduction}


Efficient exploration in reinforcement learning and black-box optimization requires principled uncertainty quantification \citep{shahriari2015taking,garnett_bayesoptbook_2023}. Ensembling multiple independently trained models is a proven, model-agnostic method to estimate predictive uncertainty by treating the variation in their predictions as a proxy for epistemic uncertainty \citep{lakshminarayanan2017simple}. However, directly applying naive ensembling to Large Language Models (LLMs) is computationally prohibitive. Furthermore, existing exploration methods for LLMs often rely on explicit reward models~\citep{lightman2024lets, coste2024reward, menet2026thompson} or depend heavily on the instruction-following capabilities of the model itself~\citep{de2025simplifying, romera2024mathematical}.

To circumvent the complexities of explicit reward modeling (see Figure~\ref{fig:method}), this work builds upon a foundational duality: policies optimized with Kullback-Leibler (KL) regularization inherently encode their underlying reward functions~\citep{ziebart2010modeling, neu2017unified, haarnoja2018soft, geist2019theory}. Following \citet{rafailov2023direct, gao2024rebel, matrenok2025quantile}, we directly match the reward function encoded by the policy to the observed data rather than learning a reward function and fitting a policy to it in two separate steps. In contrast to prior work, we use policy ensembles diversified via Poisson bootstrapping to reliably capture epistemic uncertainty and drive exploration.  The resulting algorithm is called \textbf{Po}licy \textbf{E}nsembles for \textbf{T}hompson \textbf{S}ampling (\name{}).

To overcome the prohibitive compute and memory cost of a naive LLM ensemble, we implement a parameter-efficient ``Trunk \& Branch'' architecture~\citep{lee2015mheadsbetterone}. The vast majority of weights
\begin{wrapfigure}{r}{0.65\linewidth}
    \centering
    \vspace{-7.5pt}
    \begin{subfigure}[t]{0.4\linewidth}
        \includegraphics[width=\linewidth]{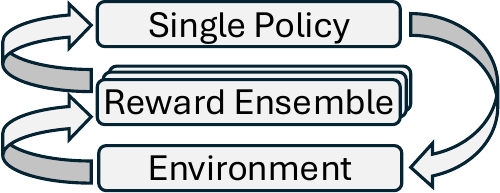}
    \end{subfigure}\hspace{20pt}
    \begin{subfigure}[t]{0.4\linewidth}
        \includegraphics[width=\linewidth]{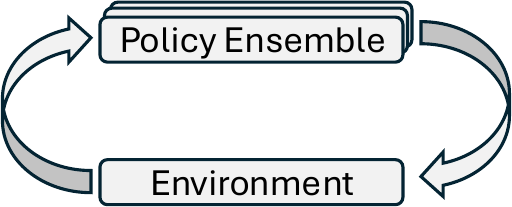}
    \end{subfigure}
    \caption{Previous methods for sample efficient RL fit an uncertainty-aware reward model and separately train a single policy to trade off exploration and exploitation (left). \name{} directly estimates reward uncertainty with policy ensembles (right).}
    \label{fig:method}
    \vspace{-15pt}
\end{wrapfigure}
are shared in a common backbone, which is only computed once per forward/backward pass. In contrast, the ensemble diversity is effectively captured via independent, low-rank adaptation (LoRA) modules~\citep{hu2022lora} applied to the final layers. Our contributions are as follows:

\begin{enumerate}
    \item We introduce \name{}, a novel framework that captures epistemic uncertainty via policy ensembles. We show that \name{} implicitly conducts approximate Thompson sampling~\citep{thompson1933likelihood, russo2018tutorial} without an explicit probabilistic reward model.
    \item We propose the first policy ensemble that scales to billion-parameter LLMs. By leveraging a specialized Trunk \& Branch architecture with intermediate gradient accumulation, running a 16-member ensemble introduces merely a $7.7\%$ runtime and $16.8\%$ memory overhead compared to a single policy instead of the usual $1600\%$ compute \& memory overhead.\looseness-1
    \item We empirically demonstrate state-of-the-art sample efficiency with \name{} across diverse combinatorial black-box optimization domains, including protein and molecular design. \looseness -1
    \item We empirically demonstrate that \name{} fundamentally improves \grpo{}'s optimization trajectories on both mathematical and coding benchmarks, unlocking the effective use of experience replay without premature overfitting, and improving model generalization.
\end{enumerate}

\section{Preliminaries}

\paragraph{Setting}
We consider a contextual bandit setting~\citep{langford2007epoch}, which is the standard reinforcement learning paradigm for LLM post-training and LLM-driven scientific discovery. Let $\mathcal{X}$ denote the discrete space of contexts (e.g., user prompts) and $\mathcal{A}$ denote the finite space of actions (i.e., generated responses of maximum length). A policy $\pi_\theta(a \mid x)$ parameterized by $\theta \in \mathbb R^{d}$ maps a context $x \in \mathcal{X}$ to a distribution over actions $a \in \mathcal{A}$. We assume access to a reward function $r: \mathcal{X} \times \mathcal{A} \to \mathbb{R}$. For a fixed context $x$, we drop the explicit conditioning and write $r(a)$ and $\pi(a)$ for brevity.\looseness-1

\paragraph{Ensembling for uncertainty quantification}
Ensembling is a highly effective, model-agnostic approach for uncertainty quantification \citep{lakshminarayanan2017simple}. By training multiple models independently with different random initializations and data bootstraps, the variance across their predictions serves as a reliable proxy for epistemic uncertainty~\citep{chua2018deep, pathak2019self, sukhija2023optimistic}. Traditionally, bootstrapping generates diverse training sets by sampling with replacement from a fixed, offline dataset of size $s$. This makes it ill-suited for the sequential, online data generation process inherent to black-box optimization and reinforcement learning without sacrificing bootstrap consistency across time. To overcome this, we follow \cite{osband2016deep} and implement \textit{Poisson bootstrapping} \citep{oza2001online, qin2013efficientonlinebootstrappinglarge}. In standard bootstrapping, the number of times a specific data point is included in a resampled dataset of size $s$ follows a binomial distribution, which asymptotically converges to a Poisson distribution: $\mathrm{Bin}(s, 1/s) \to \mathrm{Poisson}(1)$ as $s \to \infty$. Poisson bootstrapping leverages this equivalence for online learning by directly assigning each incoming data point an integer weight drawn independently from $\mathrm{Poisson}(1)$ for every ensemble member.\looseness-1

\paragraph{Regularized policies are implicit reward functions}
A foundational concept for our method is the duality between reward functions and optimal policies under regularized objectives \citep{ziebart2010modeling, haarnoja2018soft}. Consider the ``soft reward'' objective function regularized by Kullback-Leibler divergence:\looseness-1
\begin{equation}\label{eq:soft_reward_objective}
    J(\pi, r) := \mathbb{E}_{a \sim \pi}[r(a)] - \beta D_{KL}(\pi \parallel \pi_{ref}).
\end{equation}
The optimal policy that maximizes this objective takes a known closed form~\citep{geist2019theory}:
\begin{equation*}
    \pi^*(a) = \pi_{ref}(a) \exp\big(r(a)/\beta - \log Z\big).
\end{equation*}
Here, the log-partition function $\log Z := \log \sum_{a} \pi_{ref}(a) \exp(r(a)/\beta)$ ensures the policy integrates to one. By rearranging this expression, we see that any policy $\pi$ implicitly encodes a corresponding reward function $r_\pi$ for which it is optimal~\citep{rafailov2023direct}:
\begin{equation}
    r_\pi(a) := \beta \cdot (\log \pi(a) - \log \pi_{ref}(a) + \log Z). \label{eq:policies_are_reward_functions}
\end{equation}
While $\log Z$ is notoriously intractable, it acts strictly as an action-independent constant. This reflects the fundamental property that shifting a reward function by a global additive constant does not alter the optimal policy, a feature we exploit in the next section to bypass estimating $Z$ entirely.

\paragraph{Notation} For a reward function $r:\mathcal{A} \to \mathbb{R}$ and a distribution $\rho$ over actions in $\mathcal{A}$, we define the expected reward as $r(\rho) := \mathbb{E}_{a^\prime \sim \rho}[r(a^\prime)]$ and the centered reward as $r(a-\rho) := r(a)-r(\rho)$. We denote an ensemble of $n$ policies as $\Pi := \{\pi_1, \dots, \pi_n\}$ with global policy $\pi_{avg}(a) := \frac{1}{n}\sum_{i=1}^n \pi_i(a)$.\looseness-1

\section{Method}\label{sec:method}
The foundational principle of \name{} is to \emph{skip explicit reward modeling} (see Figure~\ref{fig:method}). Instead of learning a separate ensemble of reward functions and fitting a policy to it that balances exploration with exploitation, \emph{we directly match the implicit reward functions encoded by a policy ensemble to the data.} \name{} achieves this through three core components: (1) formulating an implicit reward-matching objective that reduces to a simple policy gradient, (2) capturing epistemic uncertainty via online Poisson bootstrapping, and (3) scaling to LLMs via a highly compute-efficient Trunk \& Branch architecture. \looseness -1

\subsection{Ensembling implicit reward functions via bootstrapped policy gradients}\label{sec:ensembling_implicit_reward_functions_via_bootstrapped_policy_gradients}
To formalize this, we first define the objective that aligns a single policy's implicit reward with the true reward and derive its gradients. Then, we introduce bootstrapping. All proofs are in Appendix~\ref{sec:proofs}.\looseness-1
\begin{defn}[\name{} loss]
    Define $L_{poets}^\rho(a, \pi) := \big( r(a-\rho)-r_\pi(a-\rho)\big)^2$ 
    where $\rho$ is an arbitrary baseline distribution over actions, canceling out the intractable partition function $\log Z$.
\end{defn}
\begin{prop}[\name{} gradients]\label{prop:poets_gradients} Let $\pi$ be parameterized by $\theta$ and define the KL-regularized reward $\tilde{r}_\pi(a) := r(a) + \beta \cdot (\log \pi_{ref}(a) - \log \pi(a))$. Then the gradients of the \name{} loss $L_{poets}^\rho$ are
    \begin{align}
        \nabla_\theta L_{poets}^\rho(a, \pi) & = -2\beta \cdot   \tilde{r}_\pi(a-\rho)\cdot (\nabla_\theta \log \pi(a) - \mathbb E_{a^\prime \sim \rho}[\nabla_\theta \log \pi(a^\prime)])\label{eq:poets_gradient}\\
        \mathbb E_{a \sim \rho} [\nabla_\theta L_{poets}^\rho(a, \pi)] & = -2\beta \cdot \mathbb E_{a \sim \rho} \big[\tilde r_\pi(a-\rho) \nabla_\theta \log \pi(a)\big].\label{eq:poets_as_grpo}
    \end{align}
\end{prop}
Note that if we choose $\rho$ to be the empirical distribution over entries in a batch sampled from the current policy, i.e., $\rho(a) = \tfrac{1}{G}\sum_{i=1}^G \delta_{a=a_i}$ for $a_i \sim \pi$, we obtain a surprisingly familiar result. Up to rescaling (via advantage standardization and the factor $2\beta$), Equation~\eqref{eq:poets_as_grpo} is exactly the gradient of \grpo{}. Moreover, the standard policy gradient baseline \citep{kool2019buy, ahmadian2024back} used for variance reduction emerges naturally as the anchor across which implicit reward differences are computed. Still, $\rho$ can also be off-policy, mirroring recent analysis of \grpo{}~\citep{yao2026grouprelative}.

\begin{takeaway}
    Standard \grpo{} is mathematically equivalent to optimizing a single-policy instance of the \name{} loss. By extending this implicit reward-matching objective to an ensemble, \name{} seamlessly integrates principled epistemic uncertainty quantification into the \grpo{} framework.
\end{takeaway}

\paragraph{Bootstrapping \grpo{} provides an approximate reward posterior}
Crucially, Proposition~\ref{prop:poets_gradients} proves that a policy ensemble $\Pi$ trained via Poisson-bootstrapped \grpo{} is mathematically equivalent to a bootstrapped ensemble of mean-squared-error (MSE) fitted implicit reward functions. Since bootstrapped MSE-fitted neural networks are well-calibrated and established frequentist approximations of Bayesian posteriors \citep{osband2016deep, lakshminarayanan2017simple, chua2018deep, pathak2019self, curi2020efficient, sekar2020planning, sukhija2023optimistic}, the implicit reward ensemble of \name{} inherits this posterior-approximation interpretation. In Section~\ref{sec:theory}, we build on this to show that our practical algorithm approximates Thompson sampling, a principled strategy for balancing exploration and exploitation. In Appendix~\ref{sec:correlation}, we empirically examine whether uncertainty quantification remains informative in our computationally efficient implementation, and find near-perfect Pearson correlation between the ensemble's implicit reward variance (epistemic uncertainty) and the squared error on the TDC molecular design benchmarks.\looseness-1

\paragraph{Algorithm} Transforming standard \grpo{} into \name{} requires only a few modifications to the training loop, summarized in Algorithm~\ref{alg:main}. Instead of updating a single policy, we maintain a lightweight policy ensemble $\Pi = \{\pi_1, \dots, \pi_n\}$. For each group of rollouts, we assign independent integer weights $w \sim \mathrm{Poisson}(1)$ to each response for every ensemble member. We then compute and apply the bootstrapped policy gradient (Equation~\eqref{eq:poets_as_grpo}) to each policy based on its Poisson weights. \looseness-1

\begin{algorithm}[h]
\caption{\name{} (for single context $x$)}
\label{alg:main}
\begin{algorithmic}[1]
\State \textbf{Input:} policies $\Pi=(\pi_i)_{i=1}^n$, group size $G$, buffer size $T$, regularization $\beta$, optimizer $\mathrm{Optim}$
\While{not converged}
    \State Sample actions $(a_1, \ldots, a_G) \sim \pi_{avg}$ and evaluate rewards $(r_1, \ldots, r_G)$
    \State Sample Poisson weights $(w_1^i, \ldots, w_{G}^i) \sim \mathrm{Poisson}(1)$ for each policy $i \in \{1, \ldots, n\}$
    \State Add actions to $A_{buffer}$, rewards to $R_{buffer}$, and Poisson weights to $W_{buffer}$
    
    \For{$t=1, \ldots, T$} \Comment{Experience replay over the last $T$ batches (rollout groups)}
        \State Fetch historical batch $a, r, w = A_{buffer}[t], R_{buffer}[t], W_{buffer}[t]$
        \For{$i = 1, \ldots, n$}
            \State Compute regularized rewards: $\tilde{r}_{\pi_i}(a_j) = r_j + \beta \log \pi_{ref}(a_j) - \beta \log \pi_i(a_j)$ $\forall j \in [G]$
            \State Construct data bootstrap $\rho_i(a) \propto \sum_{j=1}^G w_j^i\cdot \delta_{a=a_j}$ as empirical distribution of batch
            \State Accumulate gradients $\mathbb E_{\rho_i}[\nabla_\theta L_{poets}^{\rho_i}]$ via bootstrapped policy gradient (e.g., \grpo{})
        \EndFor
        \State Update shared (trunk) and individual (branch) parameters: $\theta \leftarrow \mathrm{Optim}(\theta, \nabla \theta)$
    \EndFor
\EndWhile
\end{algorithmic}
\end{algorithm}

In Algorithm~\ref{alg:main}, $\pi_{avg}(a) := \tfrac{1}{n} \sum_{i=1}^n \pi_i(a)$ is the approximate Thompson sampling policy (see Section~\ref{sec:theory}). To sample from it, each group rollout $a_j$ is assigned an index $i \sim \mathcal U([n])$ that identifies the policy $\pi_i$ from which it is generated (see Appendix~\ref{sec:implementation}). $A_{buffer}$, $R_{buffer}$, $W_{buffer}$ are first-in-first-out buffers of size $T$. The entirety of parameters is denoted by $\theta = (\theta_{shared}, \theta_1, \ldots, \theta_n)$. Note that the policy ensemble can share parameters, a crucial property for computational and memory efficiency.\looseness-1

\subsection{Sublinear compute \& memory scaling via Trunk \& Branch architecture}\label{sec:trunk_and_branch_architecture}
A standard implementation of policy ensembles would require maintaining $n$ independent copies of an LLM, which is prohibitively expensive in terms of both memory and compute. Instead, we adopt a parameter-efficient ``Trunk \& Branch'' \citep{lee2015mheadsbetterone} architecture using Low-Rank Adaptation (LoRA) \citep{hu2022lora, wang2023loraensembleslargelanguage}. We decompose the overall parameters $\theta$ into a shared backbone $\theta_{shared}$ (the Trunk) and $n$ independent adapter modules $\theta_i = \{A_i, B_i\}_{i=1}^n$ (the Branches).

\textbf{Architecture}
While $\theta_{shared}$ includes the vast majority of the weights, we encapsulate the ensemble diversity entirely through the final readout layer, relying on shared contextualized features~\citep{park2024the}. Note that both $\theta_{shared}$ and $\theta_i$ are fine-tuned on the data, so LoRA does not limit the capacity of single-policy fine-tuning (and thus $\pi_{avg}$). We use separate learning rates for trunk \& branches to ensure effective diversification. For a language head $h(x) = W x$, the $i$-th ensemble member computes:
\begin{equation*}
    h_i(x) = W x + B_i A_i x,
\end{equation*}
where $W \in \mathbb{R}^{d_v \times d_h}$ is a shared readout matrix, and $A_i \in \mathbb{R}^{r \times d_h}, B_i \in \mathbb{R}^{d_v \times r}$ are the low-rank decomposition matrices for branch $i$ with rank $r \ll d_h, d_v$. Maintaining $n$ branches adds only $n \times r\times (d_h + d_v)$ parameters to the $\Omega(l \times d_h^2 + d_h \times d_v)$ parameters spread across the $l$ layers of the trunk. Sampling from $\pi_{avg}$ comes at zero overhead, since for each rollout of batched generation only one policy (selected uniformly at random) is played. Moreover, because the shared trunk is executed exactly once per pass during training (via gradient accumulation), the primary computational cost is completely decoupled from the ensemble size $n$. As we confirm empirically in Section~\ref{sec:computational_overhead}, generating responses and accumulating gradients for $n$ policies takes virtually the same wall-clock time as a single policy. In contrast, prior policy ensemble methods require $n$ times more memory and computation~\citep{yang2022applicablereinforcementlearningimproving, liu2017steinvariationalpolicygradient, eysenbach2018diversity, parker2020effective}.

\begin{takeaway}
    By capturing epistemic uncertainty within a lightweight Trunk \& Branch architecture, \name{} essentially matches the memory and computational footprint of standard \grpo{}.
\end{takeaway}

\section{Theory}\label{sec:theory}
As established in Section~\ref{sec:method} and Appendix~\ref{sec:correlation}, \name{} maintains an implicit reward ensemble that acts as a frequentist approximation of the Bayesian posterior over reward functions. We now formalize how \name{} leverages this property to efficiently approximate KL-regularized Thompson sampling.\looseness-1

\begin{defn}[KL-regularized Thompson sampling]
    Given a belief over rewards, KL-regularized Thompson sampling draws a reward function $r$ and then acts according to the policy $\pi$ satisfying
    \begin{equation}
        \pi(a) \propto \pi_{ref}(a) \exp(r(a)/\beta) \iff \pi = \arg\max_{\pi^\prime} J(\pi^\prime, r).\label{eq:kl_regularized_thompson_sampling}
    \end{equation}
\end{defn}
Note that, as the regularization coefficient $\beta \to 0$, the policy concentrates entirely on the maximum reward, i.e., $\pi(a) \to \delta_{a= \arg\max_{a^\prime} r(a^\prime)}$, recovering standard non-regularized Thompson sampling. 

Crucially, each ensemble member $\pi_i$ in Algorithm~\ref{alg:main} is, by definition, the analytical solution to Equation~\eqref{eq:kl_regularized_thompson_sampling} for its corresponding implicit reward $r_{\pi_i}$. Because the ensemble $(r_{\pi_i})_{i=1}^n$ represents our approximate reward posterior, executing KL-regularized Thompson sampling reduces to drawing and playing a policy uniformly at random from the ensemble $\Pi = \{\pi_1, \dots, \pi_n\}$~\citep{lu2017ensemble}. \looseness -1

\begin{takeaway}
    \name{} implicitly approximates KL-regularized Thompson sampling by directly sampling from an uncertainty-aware policy ensemble, bypassing the need for an explicit belief over rewards.
\end{takeaway}

This connection is theoretically significant, since Thompson sampling is a provably efficient algorithm that naturally balances exploration with exploitation. In particular, under standard Bayesian assumptions---where the true reward vector is a realization of a Gaussian prior---KL-regularized Thompson sampling inherits the strong theoretical guarantees of standard Thompson sampling.

\begin{thm}\label{thm:cumulative_regret_bound}
    Sequential KL-regularized Thompson sampling has Bayesian cumulative (soft) regret
    \begin{equation*}
        \textstyle \sum_{t=1}^T \mathbb E_{r,\pi_t} [\max_{\pi}J(\pi, r) - J(\pi_t, r)] 
        \in {\mathcal O}(\sqrt{T \gamma_T})
    \end{equation*}
    where $\gamma_T$ is the maximal information gain~\citep{srinivas2009gaussian} of a Gaussian prior over reward $r$.\looseness-1
\end{thm}

Because our policies optimize a regularized objective, we bound performance using soft regret. This theorem clarifies why \grpo{} is prone to overfitting: it is the $n=1$ instance of \name{}, i.e., Thompson sampling with a single MSE-fitted implicit reward function, ceding posterior exploration.\looseness-1

\section{Experiments}\label{sec:experiments}
We empirically evaluate \name{} on black-box optimization across four diverse domains for scientific discovery as well as on reinforcement learning via verifiable rewards in the settings of math and coding. We report the mean and standard error of each metric using $25$ random seeds. Our code is available at \url{https://github.com/IBM/thompson-sampling-via-fine-tuning-of-llms} and the experimental setup is detailed in Appendix~\ref{sec:experimental_details}. Additional results and ablations are in Appendix~\ref{sec:additional_results}.\looseness-1

\subsection{\name{} for black-box optimization}
We consider the \textit{best-seen reward} of a bandit without observation noise, defined as the maximum reward observed up to step $t$, i.e., $\max_{\tau\leq t} r({a_\tau})$. This metric captures the best solution found so far and is often used in black-box optimization benchmarks. Note the connection to simple regret, given by $\max_a r(a) - \max_{\tau\leq t} r({a_\tau})$. Within a collection of methods covering standard policy gradient, in-context and gradient-based Bayesian optimization, and evolutionary search, \name{} exhibits state-of-the-art sample efficiency and computational efficiency. Moreover, sample efficiency can be improved even more by increasing the computational effort per round of Bayesian optimization via replay buffers, a technique that we demonstrate fails without ensembling. Details are in Appendix~\ref{sec:experimental_details_pure_exploration}.

\subsubsection{Baselines}
We compare against four baselines. As a strong reinforcement learning baseline, we consider an improved version of \grpo{} \citep{shao2024deepseekmath} introduced by \cite{zheng2025groupsequencepolicyoptimization} and used to post-train models in the Qwen3 family \citep{yang2025qwen3technicalreport}. As alternative methods for implementing Thompson sampling with LLMs, we consider \tosfit{} \citep{menet2026thompson} and \fibo{} \citep{de2025simplifying, menet2026codingagentsenvironmentinteraction}. Whereas \tosfit{} fine-tunes LLMs toward a variational Thompson sampling policy based on a Gaussian process reward surrogate, \fibo{} approximates Thompson sampling in-context by providing the strongest candidate solutions thus far along with their scores to the LLM. Among reward-model-based uncertainty approaches, TOSFIT represents a principled upper bound due to its calibrated probabilistic reward model. Finally, \es{} prompts a language model to conduct search via evolutionary operators such as crossover and mutation of candidates. It is a simplified implementation by \cite{menet2026thompson} of FunSearch/AlphaEvolve/OpenEvolve~\citep{romera2024mathematical, novikov2025alphaevolve, openevolve}. To verify its performance, we reproduce the FunSearch bin-packing benchmark in Appendix~\ref{sec:bin_packing}. For classical non-LLM baselines, consider Appendix~\ref{sec:classical_baselines_black_box_optimization}. All methods operate in batched mode with batch size $16$. Table~\ref{tab:qualitative_method_comparison} provides a qualitative comparison.\looseness-1

\begin{table}[ht]
    \centering
    \caption{\name{} is the first method for black-box optimization that is aware of epistemic uncertainty while neither relying on explicit reward models nor on instruction-following capabilities of LLMs.}
    \label{tab:methods}
    \begin{tabular}{l|ccccc}
    & \es{} & \fibo{} & \tosfit{} & \grpo{} & \name{} (ours) \\ \hline
    Uncertainty aware & \textcolor{Red}{no} & \textcolor{Green}{yes} & \textcolor{Green}{yes} & \textcolor{Red}{no} & \textcolor{Green}{yes} \\
    Instruction following agnostic & \textcolor{Red}{no} & \textcolor{Red}{no} & \textcolor{Green}{yes} & \textcolor{Green}{yes} & \textcolor{Green}{yes} \\
    Reward model free & \textcolor{Green}{yes} & \textcolor{Green}{yes} & \textcolor{Red}{no} & \textcolor{Green}{yes} & \textcolor{Green}{yes} \\
    Risk of policy collapse & \textcolor{Green}{low} & \textcolor{Red}{high} & \textcolor{Green}{low} & \textcolor{Red}{high} & \textcolor{Green}{low}\\ \hline
    Architectural overhead & \makecell{Multiple\\ calls} & \makecell{Long\\context} & \makecell{Reward\\ model} & None & \makecell{Policy\\ensemble}\\
    Exploration strategy & Mutation & TS & TS & Entropy & TS+Entropy\\ \hline
\end{tabular}
    \label{tab:qualitative_method_comparison}
\end{table}

\subsubsection{Settings}

\paragraph{FAQ refinement}
is a task proposed by \cite{menet2026thompson} that tests the ability to optimize text based on semantic alignment. We ask a Qwen3-8B model \citep{yang2025qwen3technicalreport} to write a frequently-asked-questions (FAQ) response. The reward is modeled as the alignment to an unknown ground-truth, judged by the Qwen3-Embedding-0.6B model \citep{qwen3embedding}. As a deep kernel for the \tosfit{} baseline, we use a linear kernel over the first 256 entries of the embeddings of Qwen3-Embedding-0.6B. The search space contains all token sequences, growing exponentially in the response length.\looseness-1

\paragraph{Protein search} is another setup from \cite{menet2026thompson} that explores the challenge of designing thermally stable proteins, a task with significant implications for industrial biotechnology. We sample amino acid sequences from ProtGPT2 \citep{ferruz2022protgpt2} (0.738B parameters) and score them according to their negative thermal instability index \citep{guruprasad1990correlation}. Two baselines, \fibo{} and \esllm{}, require instruction-tuned models. There, we use the much larger and more expensive Qwen3-8B model instead. For \tosfit{}, we follow~\cite{menet2026thompson} and adopt a GP with linear kernel over the mean token embeddings from ProtGPT2 projected to the unit sphere.\looseness-1

\paragraph{Quantum circuit design}
is the final task considered from \cite{menet2026thompson}. It entails designing quantum circuits that prepare low-energy quantum states in unknown environments. To generate 7-qubit Qiskit circuits \citep{javadiabhari2024quantumcomputingqiskit}, we use a Qwen3-8B model. As reward, we consider the negative energy of the prepared state under an unknown Hamiltonian, requiring superposition and entanglement of qubits for optimal performance. The Gaussian process feature map for \tosfit{} consists of a code validity bit as well as all two-qubit Pauli observables, which can be efficiently simulated using quantum computers. Note that for the Hamiltonian considered, the ground-truth reward is a linear map of these features, a setup tailored to benefit the \tosfit{} baseline.

\paragraph{TDC molecular design} explores multi-objective de novo molecular generation using the GP-MoLFormer model~\citep{ross2025gp}. We optimize SMILES strings to balance target-specific binding affinity, drug-likeness, and synthesizability across three therapeutic targets (GSK3$\beta$, JNK3, DRD2) \citep{huang2021therapeuticsdatacommonsmachine}, with full details of the setup and results provided in Appendix~\ref{sec:tdc_results}.

\begin{wrapfigure}{R}{0.32\linewidth}
    \centering
    \vspace{-9.5pt}
    \raisebox{20pt}{\rotatebox{90}{\footnotesize Generation Entropy}}
    \includegraphics[width=0.9\linewidth]{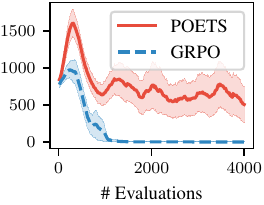}
    \caption{\name{} balances exploration with exploitation, leading to stable policy diversity. Shown on Protein Search, for other settings, see Section~\ref{sec:generation_diversity}.}
    \label{fig:diversity}
    \vspace{-20pt}
\end{wrapfigure}
\subsubsection{\name{} obtains state-of-the-art sample efficiency even without replay buffer}
\looseness -1 Consider Figure~\ref{fig:pure_exploration_best_seen_reward}, where the replay buffer is disabled by setting $T=1$ in Algorithm~\ref{alg:main}. \name{} is the only method to consistently obtain state-of-the-art sample efficiency across the highly diverse combinatorial tasks. Note that the baselines \es{} and \grpo{} conduct undirected exploration without any notion of \textit{optimism in the face of uncertainty}, i.e., an exploration bias guided by the reward potential of unexplored regions. As confirmed in Figure~\ref{fig:diversity}, the optimism of \name{} results in a more stable exploration-exploitation tradeoff than that of \grpo{}, thus managing to better retain generation diversity. 
\fibo{} and \tosfit{} are closest to \name{} in spirit, but the former relies on in-context learning instead of parameter updates and the latter requires an external reward model. As can be seen from the result, these requirements lead to less reliable performance than \name{}.

\begin{figure}[h]
\centering\includegraphics[width=\linewidth]{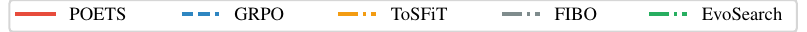}
    \raisebox{-70pt}{\rotatebox{90}{\footnotesize Best-Seen Reward}}\hspace{5pt}
    \begin{subfigure}[t]{0.3\linewidth}
        \centering\hspace{20pt}\footnotesize FAQ Refinement
        \includegraphics[width=\linewidth]{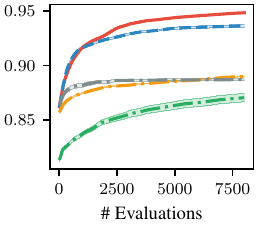}
    \end{subfigure}
    \begin{subfigure}[t]{0.3\linewidth}
        \centering\hspace{20pt}\footnotesize Protein Search
        \includegraphics[width=\linewidth]{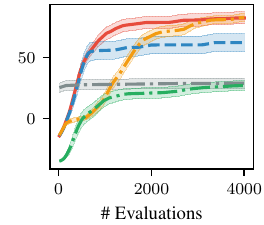}
    \end{subfigure}
    \begin{subfigure}[t]{0.3\linewidth}
        \centering\hspace{17pt}\footnotesize Quantum Circuit Design
        \includegraphics[width=\linewidth]{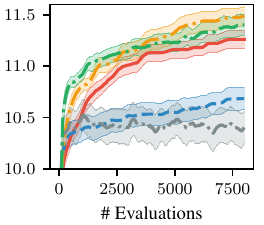}
    \end{subfigure}
    \caption{\name{} reaches SoTA sample efficiency across diverse tasks, even without a replay buffer.}
\label{fig:pure_exploration_best_seen_reward}
\end{figure}

\subsubsection{\name{} trades off compute efficiency with sample efficiency using a replay buffer}\label{sec:black_box_compute_vs_sample_efficiency}
Increasing the size of the replay buffer provides a substantial boost to sample efficiency. As demonstrated on the Protein Search task, setting the buffer size to $T>1$ in Algorithm~\ref{alg:main} yields significant improvements in both best-seen and expected reward (Figure~\ref{fig:pure_exploration_replay_buffer}). For results on quantum circuit design, see Appendix~\ref{sec:quantum_circuit_design_replay_buffer}. In contrast, standard \grpo{} is blind to epistemic uncertainty; rather than benefiting from replay, it overfits to early candidates. To understand the mechanism preventing \name{} from premature overfitting, we can track the normalized Jensen-Shannon divergence across the policy heads: $JSD(\pi_1, \ldots,\pi_n) := \frac{1}{n \log n}\sum_{i=1}^n D_{KL}(\pi_i \vert \vert \pi_{avg}) \in [0,1]$. As shown in the rightmost panel of Figure~\ref{fig:pure_exploration_replay_buffer}, the policies naturally diversify from their shared initialization. This diversity is maintained until convergence to the global optimum is achieved, at which point the policies align.\looseness-1

\begin{figure}[!ht]
\centering\includegraphics[width=\linewidth]{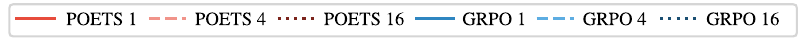}
    \begin{subfigure}[t]{0.3\linewidth}
        \centering\hspace{15pt}\scriptsize Best-Seen Reward on Protein Search\vspace{1.5pt}
        \includegraphics[width=\linewidth]{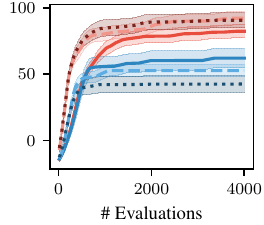}
    \end{subfigure}
    \begin{subfigure}[t]{0.3\linewidth}
        \centering\hspace{15pt}\scriptsize Expected Reward on Protein Search
        \includegraphics[width=\linewidth]{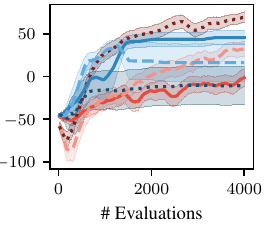}
    \end{subfigure}
    \begin{subfigure}[t]{0.3\linewidth}
        \centering\hspace{15pt}\scriptsize JS Divergence on Protein Search
        \includegraphics[width=\linewidth]{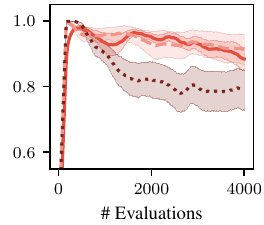}
    \end{subfigure}
    \caption{\name{} benefits substantially from a replay buffer ($T=4,16$). In contrast, \grpo{} overfits to suboptimal early solutions. After zero-initialization of LoRA, the policy heads diversify quickly.}
\label{fig:pure_exploration_replay_buffer}
\end{figure}

In addition, the JSD trajectory provides a reliable early-stopping metric for hyperparameter tuning. For example, in all our experiments we tune the LoRA learning rate by sweeping across $3$ to $4$ values between $10\times$ and $100\times$ of the shared trunk learning rate; the ideal configuration rapidly maximizes JSD before slowly decaying, a sweet spot identifiable at just 5\% of the training budget (Appendix~\ref{sec:ensemble_diversifies_despite_parameter_sharing}). Branch-only updates cannot destabilize the shared backbone, so large branch learning rates are safe.\looseness-1

\subsection{\name{} for reinforcement learning}
Next, we consider reinforcement learning with verifiable rewards (RLVR) \citep{lambert2025tulu} on the AIME dataset. In this setting, the environment provides a sparse, binary reward strictly based on whether the final answer generated within the \verb+\boxed{}+ tags matches the ground truth. Since the policy is trained to act for a distribution of questions, RLVR can be modeled as a contextual bandit problem. Adapting \name{} to this setting is straightforward: we simply add another loop over contexts in Algorithm~\ref{alg:main}. Figure~\ref{fig:reinforcement_learning_expected_reward} compares the optimization trajectories of \name{} and \grpo{} as a function of the number of rollouts/generations. We evaluate both on-policy optimization (buffer size $T=1$) and off-policy optimization with experience replay ($T=4$). As shown, standard \grpo{} is highly prone to overfitting to suboptimal solutions when exposed to repeated contexts, both in multi-epoch training (Figure~\ref{fig:reinforcement_learning_expected_reward}, left panel) and when utilizing experience replay (dashed lines).

\begin{figure}[h]
\centering\includegraphics[width=\linewidth]{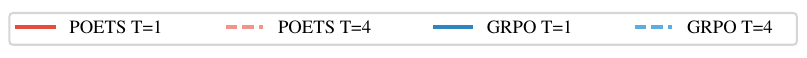}
    \raisebox{-70pt}{\rotatebox{90}{\footnotesize Reward (Accuracy)}}
    \begin{subfigure}[t]{0.29\linewidth}
        \centering\hspace{20pt}\footnotesize Training on AIME2026
        \includegraphics[width=\linewidth]{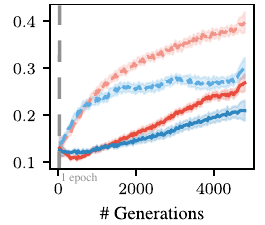}
    \end{subfigure}\hspace{5pt}
    \raisebox{-70pt}{\rotatebox{90}{\footnotesize Reward (Accuracy)}}
    \begin{subfigure}[t]{0.29\linewidth}
        \centering\hspace{15pt}\footnotesize Training on AIME1983-2024
        \includegraphics[width=\linewidth]{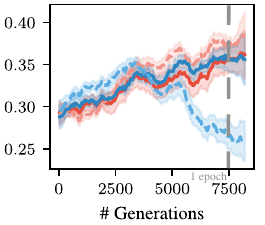}
    \end{subfigure}\hspace{5pt}
    \raisebox{-80pt}{\rotatebox{90}{\footnotesize Accuracy on AIME2026}}
    \begin{subfigure}[t]{0.29\linewidth}
        \centering\hspace{15pt}\footnotesize Training on AIME1983-2024
        \includegraphics[width=\linewidth]{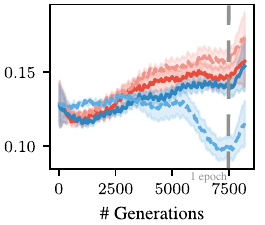}
    \end{subfigure}
    \caption{\name{} improves the sample efficiency and robustness of RL both on-policy and off-policy. Using a Qwen3-8B model, rollouts are generated in groups of $8$ for $4$ questions simultaneously.}
\label{fig:reinforcement_learning_expected_reward}
\end{figure}

In contrast, \name{} successfully mitigates these failures by implicitly realizing Thompson sampling. Because the policy ensemble maintains awareness of epistemic uncertainty, it naturally balances exploration and exploitation. This demonstrates that principled uncertainty quantification becomes essential as the risk of overfitting increases due to context repetition.

Finally, we consider reinforcement learning on LiveCodeBench v6, a subset of LiveCodeBench~\citep{jain2025livecodebench} containing 175 contest-style coding problems released between January and April 2025. In this setting, the environment provides dense rewards based on the fraction of passed unit tests, averaging 2.65 public tests for training and 37.35 private tests for evaluation. We report on the (public test \& private test) accuracy as the fraction of codes that pass all tests. In contrast to AIME, where we follow standard \grpo{} best-practice with a unified learning rate, we now modify \grpo{} and consider separate learning rates for the main model parameters $\theta_{shared}$ and a LoRA subspace of the readout layer, mirroring the architecture of our compute-efficient \name{}. As demonstrated by Figure~\ref{fig:live_code_bench_v6}, \grpo{} overfits to the public tests given the additional faster-learning LoRA readout head, while the inherent uncertainty quantification of \name{} results in a reliable transfer from public test performance to private test performance. The improved robustness to changes in learning rates is also mirrored in our results on Therapeutics Data Commons in Appendix~\ref{sec:tdc_results}. Moreover, we also observe improved OOD transfer on the task of bin-packing heuristic search, see Appendix~\ref{sec:bin_packing}.\looseness-1

\begin{figure}[t] 
    \begin{subfigure}[t]{0.29\linewidth}
        \vspace{13pt}
        \includegraphics[width=\linewidth]{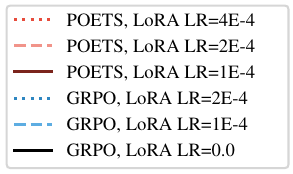}
    \end{subfigure}
    \begin{subfigure}[t]{0.35\linewidth}
        \centering\hspace{15pt}\footnotesize Public Test Accuracy
        \includegraphics[width=\linewidth]{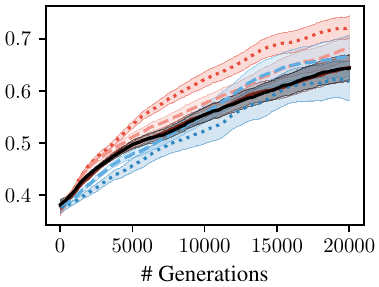}
    \end{subfigure}
    \begin{subfigure}[t]{0.36\linewidth}
        \centering\hspace{15pt}\footnotesize Private Test Accuracy
        \includegraphics[width=\linewidth]{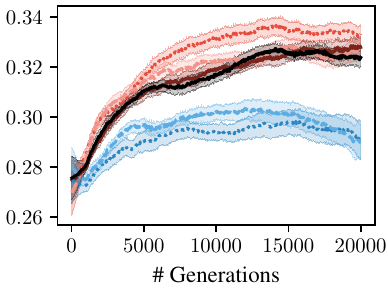}
    \end{subfigure}
    \caption{We train a Qwen3-8B model on the public tests of LiveCodeBench v6. Whereas \grpo{} overfits to the public tests, the uncertainty-aware optimization of \name{} ensures reliable transfer from public to private tests. The trunk $\theta_{shared}$ is always trained under learning rate $10^{-5}$, rollouts are generated in groups of $8$ for $4$ questions simultaneously, and a replay buffer of size T=4 is used. \looseness -1}
\label{fig:live_code_bench_v6}
\end{figure}

\begin{takeaway}
    \name{} consistently improves optimization trajectories of \grpo{}, achieving state-of-the-art black-box optimization, more robust off-policy reinforcement learning, and less reward hacking.
\end{takeaway}

\subsection{Computational overhead}\label{sec:computational_overhead}

Due to the low overhead of our last-layer ``Trunk \& Branch'' LoRA implementation, the benefits of \name{} come at minimal computational cost. As detailed in Table~\ref{tab:n_policies_vs_runtime}, training Qwen3-8B on AIME1983-2024 with a replay buffer of $T=4$ on a single NVIDIA H200 \textbf{introduces a mere $\mathbf{7.7\%}$ total runtime overhead} compared to standard \grpo{}. To cut costs further, note that the ensemble is already fully diversified and would achieve equivalent accuracy at $n=8$ (Figure~\ref{fig:ablate_n_ensembles} in Appendix~\ref{sec:additional_results}). 

\begin{table}[htbp]
\centering
\caption{Our Trunk \& Branch implementation adds only a small overhead compared to a single policy.\looseness-1}
\resizebox{\linewidth}{!}{
\begin{tabular}{l|ccccc}
 & $n=1$ & $n=2$ & $n=4$ & $n=8$ & $n=16$ (default) \\
\hline
\rule{0pt}{10pt}Peak VRAM (GB) & $116.0$ & $116.8^{+0.7\%}$ & $118.0^{+1.7\%}$ & $124.9^{+7.7\%}$ & $135.5^{+16.8\%}$ \\
\hline
\rule{0pt}{10pt}AR Generation (s / round) & $\ \ 81.4$ & $\ \ 80.7^{-0.9\%}$ & $\ \ 81.4^{+0.0\%}$ & $\ \ 81.1^{-0.4\%}$ & $\ \ 81.0^{-0.5\%}$ \\

\rule{0pt}{10pt}Parallel Training (s / round) & $\ \ 36.6$ & $\ \ 37.2^{+1.6\%}$ & $\ \ 38.6^{+5.5\%}$ & $\ \ \ \, 41.1^{+12.3\%}$ & $\ \ \ \, 46.1^{+26.0\%}$ \\
\hline
\begin{tabular}{@{}l@{}}\rule{0pt}{9pt}AR Generation + \\ Parallel Training (s / round)\end{tabular} & $118.0$ & $117.9^{+0.0\%}$ & $120.0^{+1.7\%}$ & $122.2^{+3.6\%}$ & $127.1^{+7.7\%}$
\end{tabular}
}
\label{tab:n_policies_vs_runtime}
\end{table}

\section{Related work}

\subsection{Large language models for optimization and scientific discovery}
Applying LLMs to complex search spaces like protein design \citep{ferruz2022protgpt2, lin2023evolutionary} and quantum circuit synthesis \citep{javadiabhari2024quantumcomputingqiskit, vishwakarma2024qiskit} is a rapidly growing area \citep{kristiadi2024sober, rankovic2025gollum}. Existing approaches leverage LLMs as evolutionary operators \citep{romera2024mathematical}, or implement Thompson sampling \citep{thompson1933likelihood, russo2018tutorial} either via in-context conditioning (\fibo{}) \citep{de2025simplifying, menet2026codingagentsenvironmentinteraction} or via fine-tuning based on an explicit reward model (\tosfit{}) \citep{menet2026thompson}. Unlike these methods, \name{} does not rely on fragile in-context instruction following or additional external reward models, allowing it to fully leverage domain-specific pre-trained models.

\subsection{Policy gradient methods for large language model post-training}
\grpo{} \citep{shao2024deepseekmath} and its variants \citep{zheng2025groupsequencepolicyoptimization, liu2025understanding, yu2025dapo} have emerged as strong, critic-free RL baselines for post-training LLMs \citep{deepseekai2025deepseekv3technicalreport, guo_2025, yang2025qwen3technicalreport}. However, standard policy gradient methods \citep{williams1992simple, schulman2017proximal} such as \grpo{} conduct undirected exploration and are blind to epistemic uncertainty. \name{} introduces a simple structural modification to policy gradient methods that captures this uncertainty, stabilizing the exploration-exploitation tradeoff without any explicit external reward model.

\subsection{Ensembles for uncertainty quantification and targeted exploration}
Deep ensembles are a popular method for uncertainty quantification \citep{lakshminarayanan2017simple, yang2026rewarduqunifiedframeworkuncertaintyaware}, previously applied to Thompson sampling by fitting independent reward models to bootstrapped data \citep{eckles2014thompsonsamplingonlinebootstrap, osband2016deep, lu2017ensemble}. Instead of reward ensembles, policy ensembles have also been considered in prior works, where diversity is induced via additional loss terms~\citep{yang2022applicablereinforcementlearningimproving}, e.g., based on Stein discrepancies \citep{liu2017steinvariationalpolicygradient}, determinantal diversity~\citep{parker2020effective}, or mutual information~\citep{eysenbach2018diversity}. Alternative uncertainty mechanisms such as dropout-based uncertainty \citep{gal2016dropout} are computationally convenient but provide a weak proxy for epistemic uncertainty, failing to exhibit the desired progression from high initial uncertainty to confident convergence \citep{osband2016deep}. \looseness -1

These prior methods for policy ensembling force exploration indiscriminately, wasting compute on known low-reward regions. In contrast, \name{} grounds policy diversity strictly in epistemic uncertainty via online Poisson bootstrapping \citep{osband2016deep}. Beyond these algorithmic differences, prior policy ensemble methods require fully independent networks, resulting in $n$ times more memory and compute for ensembles of size $n$, which is computationally prohibitive for LLMs. By combining a shared backbone with independent Low-Rank Adaptation \citep{hu2022lora} heads via intermediate gradient accumulation, \name{} introduces the first policy ensemble that scales to billion-parameter LLMs, capturing the diversity required for effective exploration \citep{russo2014learning, russo2016information} while barely exceeding the memory and compute footprint of a single policy.

\section{Conclusion}\label{sec:conclusion_and_limitations}
In this work, we introduce \name{} (\textbf{Po}licy \textbf{E}nsembles for \textbf{T}hompson \textbf{S}ampling), a framework that enables principled, uncertainty-aware exploration for scientific discovery and reinforcement learning. Based on a well-established duality between regularized policies and implicit reward functions, \name{} approximates Thompson sampling without any explicit reward belief. It thus sidesteps the nested optimization procedure associated with separately training a reward model and an acting policy.\looseness-1

To overcome the prohibitive computational cost of ensembling LLMs, we deploy a ``Trunk \& Branch'' architecture. By sharing the pre-trained backbone, confining ensemble diversity to independent last-layer LoRA adapters, and accumulating gradients, \name{} achieves the exploration benefits of deep ensembles while maintaining a memory and computational footprint comparable to standard \grpo{}.\looseness-1

Theoretically, we cast \name{} as approximate Thompson sampling, which is known to achieve an optimal cumulative regret bound of ${\mathcal O}(\sqrt{T\gamma_T})$~\citep{scarlett17a}. Empirically, \name{} demonstrates state-of-the-art sample efficiency across diverse black-box optimization tasks, including FAQ refinement, protein search, therapeutic molecular design, and quantum circuit design. Furthermore, in reinforcement learning settings such as mathematical reasoning and code generation, \name{} balances the exploration-exploitation tradeoff to maintain generation diversity and navigate unknown but promising regions. This approach improves the sample efficiency of \grpo{} in both on-policy and off-policy settings, and enables the effective use of experience replay without premature overfitting.

\textbf{Limitations \& Future Work:} 
Theorem~\ref{thm:cumulative_regret_bound} characterizes the regret of exact KL-regularized Thompson sampling; the $\mathcal{O}(\sqrt{T\gamma_T})$ bound therefore does not constitute a formal regret guarantee for \name{} itself. Nevertheless, the theorem serves as theoretical motivation: it provides a principled explanation for why maintaining a reward posterior can improve exploration over single-policy methods such as \grpo{}, which lack posterior exploration. Quantifying the approximation error introduced by the Poisson bootstrapped ensemble in order to derive formal guarantees for \name{} is an interesting open problem, which we leave to future work. Although our efficient Trunk \& Branch LoRA implementation significantly improves computational efficiency, it may also reduce the approximation quality of the implicit reward posterior compared to a standard ensemble, and requires some tuning of the branch learning rate (Section~\ref{sec:black_box_compute_vs_sample_efficiency}) to ensure sufficient diversification. These are deliberate tradeoffs and the key ingredients that make \name{} the first policy ensemble that scales to billion-parameter LLMs. Future research directions include extending \name{} from the contextual bandit setting of LLM post-training and scientific discovery to multi-step settings with step-wise rewards. This is non-trivial, since the log partition function $\log Z$ becomes state-dependent and can no longer be easily canceled. By offering a scalable and model-agnostic approach to exploration, \name{} provides a foundation for applying LLMs to complex sequential decision-making and scientific discovery. \looseness-1


\section*{Acknowledgment}
This work is supported by the Swiss National Science Foundation (SNSF), grant 10002666. We thank Jonas Hübotter and Simon Matrenok for providing feedback on an initial version of the paper.

\bibliographystyle{plainnat}
\bibliography{bibliography}

\appendix\newpage

\section{Additional experiments and ablations}\label{sec:additional_results}

\subsection{\name{} learns a calibrated reward ensemble}\label{sec:correlation}

\begin{wrapfigure}{r}{0.33\linewidth}
    \centering
    \vspace{-18.5pt}
    
    \includegraphics[width=0.9\linewidth]{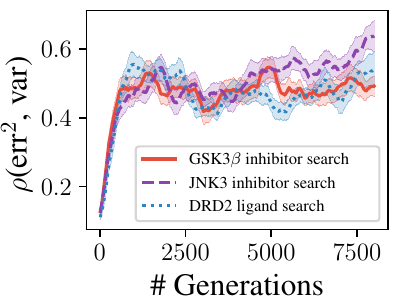}
    \caption{The implicit reward ensemble of \name{} achieves calibrated uncertainty estimation.}
    \label{fig:calibration_correlation}
    \vspace{-3pt}
\end{wrapfigure}
As established in Section~\ref{sec:method}, \name{} optimizes a policy ensemble that inherently encodes an ensemble of implicit reward functions. To justify the connection to Thompson sampling and thus Theorem~\ref{thm:cumulative_regret_bound}, we now provide evidence that the implicit reward ensemble is well-calibrated and thus a strong frequentist approximation to the Bayesian reward posterior. Because the intractable log-partition function $\log Z$ acts strictly as an action-independent constant, the implicit rewards are only defined up to a global offset. To bypass this constant, our \grpo{}-based objective (Equation~\ref{eq:poets_as_grpo}) anchors the optimization to an empirical baseline, specifically the average reward of the rollout group. Consequently, \name{} effectively learns an approximate posterior over \emph{implicit advantages} (reward differences relative to a baseline) rather than absolute rewards.
To verify that this implicit advantage ensemble is well-calibrated, we evaluate it on the Therapeutics Data Commons (TDC) molecular design benchmarks. For each generated candidate, we compute the Pearson correlation $\rho(\text{err}^2, \text{var})$ between the ensemble's variance (the epistemic uncertainty over the implicit advantages) and the actual squared error (the squared difference between the ensemble's mean implicit advantage and the true oracle advantage). For perfectly well-calibrated ensembles, a Pearson correlation of at most $57.7\%$ can be reached (see Proposition~\ref{prop:calibration_correlation_bound}). As shown in Figure~\ref{fig:calibration_correlation}, \name{} yields a near-optimally calibrated posterior: the Pearson correlation averages $47\%$ for GSK3$\beta$, $50\%$ for JNK3, and $47\%$ for DRD2.\looseness-1
\begin{prop}[Pearson correlation for perfect Gaussian calibration]\label{prop:calibration_correlation_bound} Let $\mathrm{var} > 0$ be a random variable representing the predicted variance, and let $\mathrm{err}$ be the signed error, such that $\mathrm{err}^2$ is the squared error. Assume the model is perfectly calibrated under Gaussian assumptions, i.e., $\mathrm{err}^2 = \mathrm{var} \cdot Z^2$ for $Z \sim \mathcal{N}(0, 1)$ independent of $\mathrm{var}$. Then, the Pearson correlation $\rho(\mathrm{err}^2, \mathrm{var})$ is bounded from above:\looseness-1 $$ \rho(\mathrm{err}^2, \mathrm{var}) \leq \frac{1}{\sqrt{3}} \approx 0.577 $$ \end{prop}

\subsection{Policy ensembling ensures generation diversity in black-box optimization}\label{sec:generation_diversity}

A longstanding challenge of policy gradient methods is entropy collapse~\citep{mei2020global}. By considering epistemic uncertainty, \name{} is naturally robust to such degeneracies, see Figure~\ref{fig:pure_exploration_entropy_comparison}. Notably, \name{} does not rely on strong entropy regularization~\citep{mnih2016asynchronous}, which affects the reward landscape and thus changes the optimal policy~\citep{haarnoja2018soft, ahmed2019understanding}.

\begin{figure}[h]
    \centering
    \raisebox{-73pt}{\rotatebox{90}{\footnotesize Generation Entropy}}\hfill
    \begin{subfigure}[t]{0.32\linewidth}
    \centering\hspace{20pt}\footnotesize FAQ Refinement
        \includegraphics[width=\linewidth]{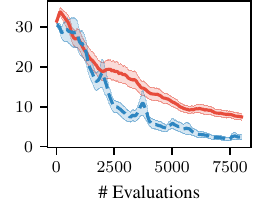}
    \end{subfigure}
    \begin{subfigure}[t]{0.32\linewidth}
        \centering\hspace{20pt}\footnotesize Protein Search
        \includegraphics[width=\linewidth]{graphics/protein_entropy.pdf}
    \end{subfigure}
    \begin{subfigure}[t]{0.32\linewidth}
        \centering\hspace{20pt}\footnotesize Quantum Circuit Design
        \includegraphics[width=\linewidth]{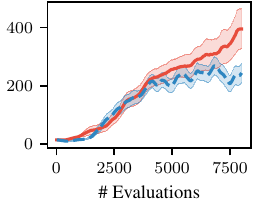}
    \end{subfigure}
    \caption{\name{} does not suffer from the diversity collapse of \grpo{}. Note that we have added entropy regularization to the quantum circuit design task, since the starting entropy is extremely low.}
    \label{fig:pure_exploration_entropy_comparison}
\end{figure}

We note that since rollout generation and model training share the same HuggingFace Transformers stack, the importance sampling ratio is exactly $1.0$ in the on-policy setting (T=1), meaning importance clipping is inactive. DAPO's~\citep{yu2025dapo} clip-higher design to avoid diversity collapse is thus inapplicable (and unnecessary) for the majority of our experimental configurations, which is why we universally default to the well-established GSPO~\citep{zheng2025groupsequencepolicyoptimization} variant of \grpo{}~\citep{shao2024deepseekmath}.\looseness-1

\subsection{The policy ensemble diversifies effectively despite parameter sharing}\label{sec:ensemble_diversifies_despite_parameter_sharing}
Independent ensembles are known to diversify naturally~\citep{lakshminarayanan2017simple}. However, given the trend in up-scaling LLMs, running a forward/backward pass for $n$ independent models would be prohibitive both in terms of compute and memory. Our proposed solution in Section~\ref{sec:trunk_and_branch_architecture} is to use a single LLM as a shared backbone, and then construct a policy ensemble through LoRA head adapters. While computationally efficient, the tying of weights between separate policies in the ensemble risks diversity collapse across ensemble members. As a remedy, we propose to use separate learning rates for the shared backbone and the LoRA heads. Consider our ablation on the LoRA head learning rate in Figure~\ref{fig:ablation_fast_weights} using Qwen3-8B with $T=4$, where the backbone learning rate was selected via grid-search to $2.5\times 10^{-6}$ while keeping the ensemble learning rate fixed to $0$ (see Section~\ref{sec:experimental_details_reinforcement_learning}). 

\begin{figure}[h]
    \centering
    \raisebox{-99pt}{\rotatebox{90}{\footnotesize Jensen-Shannon Divergence}}
    \begin{subfigure}[t]{0.39\linewidth}
    \centering\hspace{15pt}\footnotesize Qwen3-8B Trained on AIME1983-2024
        \includegraphics[width=\linewidth]{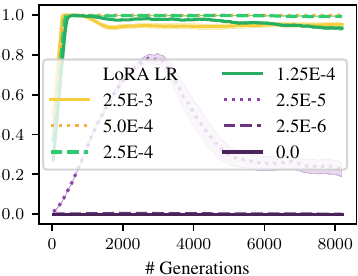}
    \end{subfigure}
    \hspace{20pt}
    \raisebox{-94pt}{\rotatebox{90}{\footnotesize Accuracy on AIME2026}}
    \begin{subfigure}[t]{0.4\linewidth}
        \centering\hspace{20pt}\footnotesize Qwen3-8B Trained on AIME1983-2024
        \includegraphics[width=\linewidth]{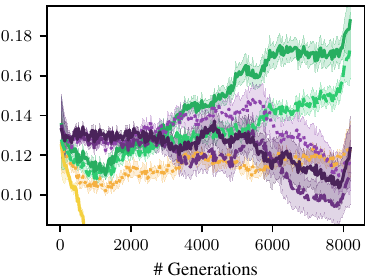}
    \end{subfigure}
    \caption{By selecting the correct learning rate for the LoRA heads of \name{}, the policy ensemble is effectively diversified despite adopting a shared backbone (the trunk) for training efficiency.}
    \label{fig:ablation_fast_weights}
\end{figure}

We remark that the best performing configuration has the property that the (normalized) Jensen-Shannon divergence across the policy ensemble approximately reaches its maximum value of $1$ after just $5\%$ of the training budget, but then slowly decays, thereby allowing the ensemble to gradually converge to a single optimal policy. This pattern of rapid diversification followed by slow convergence consistently leads to the best trajectories across all our experiments. As such, LoRA learning rates that strike a good balance between exploration and exploitation can be identified early in training.

\begin{wrapfigure}{R}{0.42\linewidth}
    \centering
    \vspace{-5pt}

    \raisebox{-30pt}{\rotatebox{90}{\footnotesize Accuracy on AIME2026}}
    \begin{minipage}{0.9\linewidth}
    \centering\hspace{20pt}\footnotesize Qwen3-8B Trained on AIME1983-2024
    \includegraphics[width=\linewidth]{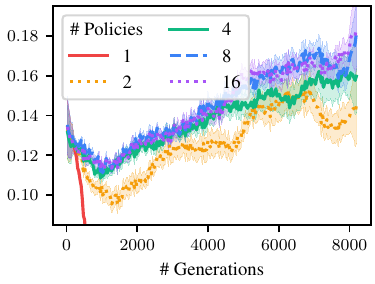}
    \end{minipage}
    \caption{Increasing the ensemble size results in improved performance. At $8$ members, the ensemble is sufficiently diverse.}
    \label{fig:ablate_n_ensembles}
    \vspace{-10pt}
\end{wrapfigure}

\subsection{Small ensembles suffice to effectively capture the reward posterior}
In Figure~\ref{fig:ablate_n_ensembles}, we ablate the number of policy heads $n$ that make up the ensemble. Recall that each policy implicitly encodes a reward function, as such, the size of the ensemble directly affects the fidelity of the probabilistic reward model. While larger policy ensembles improve performance, the marginal improvement vanishes for ensembles larger than $8$. Unless otherwise stated, we adopt an ensemble size $16$ in all experiments.

Setting the maximum generation length to $2048$ tokens, autoregressively generating from $\pi_{avg}$ $8$ rollouts for each of the $4$ task descriptions requires $81.1$ seconds. Note that, according to Algorithm~\ref{alg:main}, rollout cost does not depend on the ensemble size. In contrast, training does. However, using our efficient ``Trunk \& Branch'' implementation, the computational overhead is minor at $7.7\%$ in total. Table~\ref{tab:n_policies_vs_runtime} provides detailed measurements.

\subsection{Large LoRA ranks improve ensemble diversification}
Next, we consider the impact of the LoRA rank on the policy ensemble diversity and thus the optimization trajectory.
\begin{figure}[h]
    \centering
    \raisebox{-98pt}{\rotatebox{90}{\footnotesize Jensen-Shannon Divergence}}
    \begin{subfigure}[t]{0.39\linewidth}
        \centering\hspace{18pt}\footnotesize Qwen3-8B Trained on AIME1983-2024
        \includegraphics[width=\linewidth]{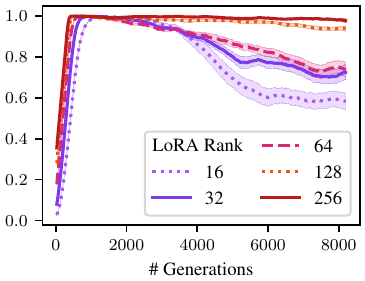}
    \end{subfigure}\hspace{20pt}
    \raisebox{-90pt}{\rotatebox{90}{\footnotesize Accuracy on AIME2026}}
    \begin{subfigure}[t]{0.38\linewidth}
        \centering\hspace{18pt}\footnotesize Qwen3-8B Trained on AIME1983-2024
        \includegraphics[width=\linewidth]{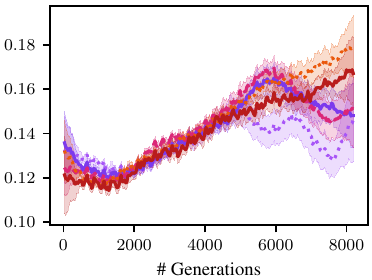}
    \end{subfigure}
    \caption{Increasing the rank of the last-layer LoRA heads increases the diversity of the policy ensemble and thus improves the optimization trajectory. Here, the replay buffer is set to $T=4$.}
    \label{fig:ablate_on_lora_rank}
\end{figure}
As can be seen in Figure~\ref{fig:ablate_on_lora_rank}, increasing the LoRA rank from 16 to 256 provides consistent gains on the ensemble diversity, as measured by the Jensen-Shannon divergence. Moreover, the best optimization trajectories are achieved by ranks of $128$ and $256$, demonstrating the utility of large ranks. Note that in the rest of the paper, we consistently adopt the LoRA rank $128$, which only results in a small computational overhead compared to single-policy~\grpo{}, see Table~\ref{tab:n_policies_vs_runtime}.

\begin{wrapfigure}{R}{0.45\linewidth}
    \centering

    \vspace{-8pt}
    \raisebox{-33pt}{\rotatebox{90}{\footnotesize Jensen-Shannon Divergence}}
    \begin{minipage}{0.9\linewidth}
    \centering\hspace{25pt}\footnotesize ProtGPT2 on Protein Search
    \includegraphics[width=\linewidth]{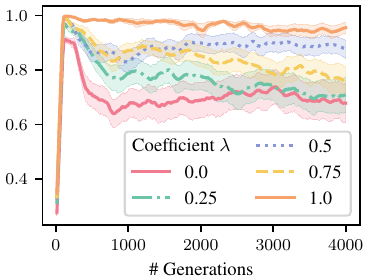}
    \end{minipage}
    \caption{Bootstrapping ensures a diversified ensemble that captures epistemic uncertainty.}
    \label{fig:ablate_bootstrapping}
    \vspace{-15pt}
\end{wrapfigure}

\subsection{Bootstrapping contributes to policy diversity}
In Figure~\ref{fig:ablate_bootstrapping}, we investigate whether the Poisson bootstrapping of Algorithm~\ref{alg:main} is really required to ensure a diversified policy ensemble. To that end, we ablate the Poisson weights $w_1^i, \ldots, w_G^i \sim \mathrm{Poisson}(1)$ drawn independently for each policy $i \in \{1, \ldots, n\}$. Note that if $w_j^i = 1\ \forall i,j$, no bootstrapping would occur and all policy heads would see the same data. In that case, diversification would rely entirely on the random initialization of the $A_i$ matrices ($B_i$ is initialized to zero). In our ablation, we transform the Poisson weights via $(w_j^i)^\prime \gets \lambda w_j^i + 1 - \lambda$. In other words, we take a convex combination between bootstrapping ($\lambda = 1$) and no bootstrapping ($\lambda = 0$). As Figure~\ref{fig:ablate_bootstrapping} demonstrates, bootstrapping is a key component that ensures a diversified policy ensemble.

\subsection{KL and entropy regularization cannot achieve the targeted exploration of \name{}}\label{sec:kl_and_entropy_regularization_ablation}

In the main text, we present \name{} strictly using Kullback-Leibler (KL) regularization to simplify the exposition. However, in environments with extremely low initial entropy—such as quantum circuit design—it is common practice to introduce an explicit entropy regularization term $\alpha H[\pi]$ to prevent premature policy collapse, where $H[\pi] := \mathbb E_{a \sim \pi}[-\log \pi(a)]$ denotes the Shannon entropy.

Mathematically, optimizing an objective with both KL and entropy regularization is equivalent to optimizing a purely KL-regularized objective using a tempered reference distribution and a modified penalty coefficient. Specifically, we can rewrite the combined objective as:
$$\mathbb{E}_{a \sim \pi}[r(a)] - \beta D_{KL}(\pi \parallel \pi_{ref}) + \alpha H[\pi] = \mathbb{E}_{a \sim \pi}[r(a)] - \beta^\prime D_{KL}(\pi \parallel \pi_{ref}^\prime) + C$$
where $\beta^\prime = \beta + \alpha$, $\pi_{ref}^\prime(a) \propto \pi_{ref}(a)^{\frac{\beta}{\beta+\alpha}}$, and $C$ is a constant independent of $\pi$. While this shows that explicit entropy regularization can always be absorbed into the standard \name{} formulation, in this section we consider entropy and KL-regularization separately to facilitate interpretable ablations.

\begin{figure}[h]
    \centering
    \begin{subfigure}[t]{0.45\linewidth}
        \raisebox{50pt}{\rotatebox{90}{\footnotesize Best-Seen Reward}}\includegraphics[width=0.95\linewidth]{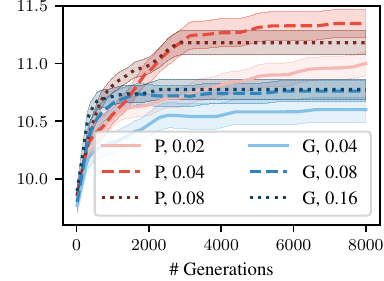}
        \subcaption[]{P and G denote \name{} and \grpo{}. We ablate the entropy regularization coefficient $\alpha$ for both.}
    \end{subfigure}\hspace{20pt}
    \begin{subfigure}[t]{0.44\linewidth}
        \raisebox{50pt}{\rotatebox{90}{\footnotesize Best-Seen Reward}}\hspace{4pt}\includegraphics[width=0.95\linewidth]{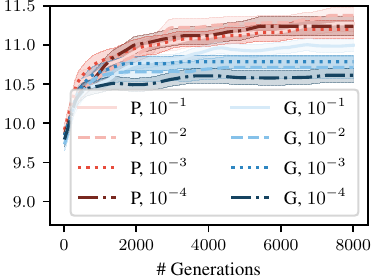}
        \subcaption[]{P and G denote \name{} and \grpo{}. We ablate the Kullback-Leibler regularization coefficient $\beta$.}
    \end{subfigure}
    \caption{Entropy/KL-regularization does not achieve the well-calibrated exploration-exploitation tradeoff inherent to \name{}. This is expected due to its Thompson sampling nature, see Section~\ref{sec:theory}.}
    \label{fig:ablate_on_alpha_beta}
\end{figure}

We ablate the hyper-parameters $\alpha$ and $\beta$ on the quantum circuit design task using a Qwen3-8B. KL and entropy regularization are a staple of reinforcement learning used to prevent policy collapse. However, neither take into account epistemic uncertainty, i.e., a progression of large initial uncertainty for exploration followed by convergence to the optimal policy for exploitation. As we demonstrate in Figure~\ref{fig:ablate_on_alpha_beta}, \name{} introduces a calibrated and adaptive exploration-exploitation tradeoff that cannot be achieved via Kullback-Leibler or entropy regularization. However, such targeted exploration is precisely needed to efficiently identify the global optimum in black-box optimization tasks.

\subsection{Replay buffer for quantum circuit design}\label{sec:quantum_circuit_design_replay_buffer}
\begin{wrapfigure}{R}{0.5\linewidth}
    \centering

    \vspace{-38pt}
    \raisebox{-20pt}{\rotatebox{90}{\footnotesize Best-Seen Reward}}
    \begin{minipage}{0.8\linewidth}
    \centering\hspace{25pt}\footnotesize Qwen3-8B on Quantum Circuit Design
    \includegraphics[width=\linewidth]{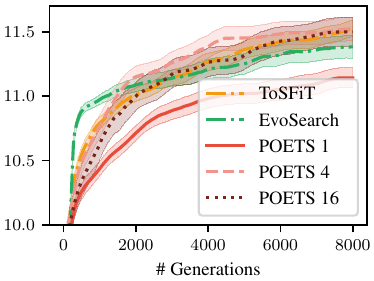}
    \end{minipage}
    \caption{Upsizing the replay buffer results in SoTA sample efficiency for quantum circuit design.}
    \label{fig:quantum_circuit_design_replay_buffer}
    \vspace{-12pt}
\end{wrapfigure}
Recall from Section~\ref{sec:black_box_compute_vs_sample_efficiency} that increasing the replay buffer provides consistent improvements in sample efficiency on the protein search task. We now repeat this experiment for the task of quantum circuit design. As can be seen in Figure~\ref{fig:quantum_circuit_design_replay_buffer}, given a replay buffer of size $4$ or $16$, \name{} recovers the SoTA sample-efficiency of \tosfit{}. Note that \tosfit{} leverages a kernel specifically handcrafted for this task, i.e., the ground-truth reward function is linear in the features. In contrast, \name{} is completely task agnostic, i.e., does not need any task-dependent engineering.\looseness-1

\subsection{Black-box optimization on Therapeutics Data Commons}\label{sec:tdc_results}

\paragraph{Therapeutics Data Commons (TDC) molecular design}
explores de novo molecular design under scalarized multi-objective rewards, balancing target-specific binding affinity against pharmacokinetic and synthesizability criteria. Once again, the search space is combinatorially large. We optimize molecular SMILES strings generated by GP-MoLFormer (47M parameters) \citep{ross2025gp} across three distinct therapeutic targets from the TDC benchmark collection \citep{huang2021therapeuticsdatacommonsmachine}: (i) \textit{GSK3$\beta$ inhibitor search} (regulates key signaling pathways in neurodegeneration and diabetes), (ii) \textit{JNK3 inhibitor search} (mediates stress-induced neuronal apoptosis), and (iii) \textit{DRD2 ligand search} (GPCR target linked to neuropsychiatric disorders). To enforce drug-likeness constraints, we optimize a composite multiplicative reward: $\text{Reward} = \text{Score}_{\text{Target}} \times \text{QED} \times \text{SA}_\text{Norm}$, where the target score is the predicted binding probability from TDC oracles (two pre-trained random forests and an SVM acting as wetlab surrogates), QED is the Quantitative Estimate of Drug-likeness measuring bioavailability, and $\text{SA}_\text{Norm}$ is the normalized Synthetic Accessibility score~\citep{ertl2009estimation}. The latter two are computed via RDKit~\citep{rdkit}. All reward components are normalized to $[0, 1]$.

\begin{figure}[h]

\centering\includegraphics[width=\linewidth]{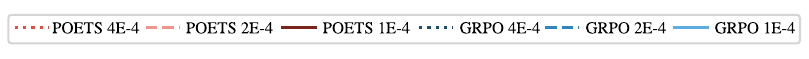}
    \raisebox{-68pt}{\rotatebox{90}{\footnotesize Best-Seen Reward}}
    \begin{subfigure}[t]{0.31\linewidth}
        \centering\hspace{20pt}\footnotesize GSK3$\beta$ inhibitor search
        \includegraphics[width=\linewidth]{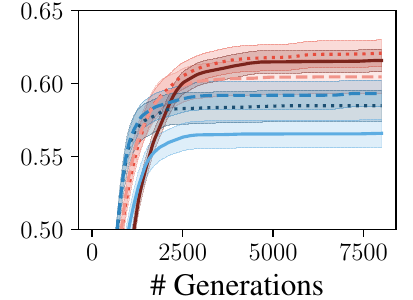}
    \end{subfigure}\hspace{5pt}
    \begin{subfigure}[t]{0.31\linewidth}
        \centering\hspace{20pt}\footnotesize JNK3 inhibitor search\vspace{2pt}
        
        \includegraphics[width=\linewidth]{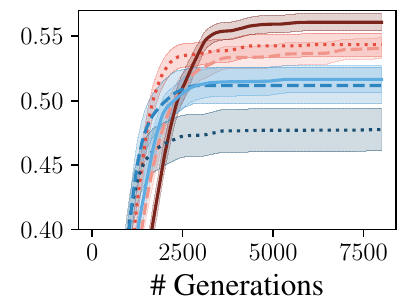}
    \end{subfigure}\hspace{5pt}
    \begin{subfigure}[t]{0.31\linewidth}
        \centering\hspace{15pt}\footnotesize DRD2 ligand search
        \includegraphics[width=\linewidth]{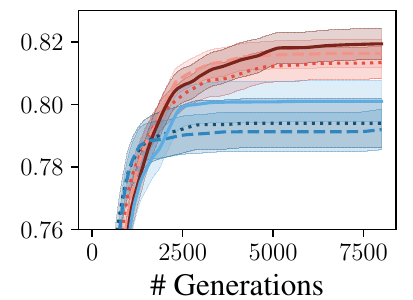}
    \end{subfigure}
    
    \caption{\name{} outperforms \grpo{} on therapeutic molecular design and is more robust to learning rates. Here, only the readout layer is trained with the learning rate indicated to the right.}
\label{fig:therapeutics_data_commons_best_seen_reward}
\end{figure}

Policy gradient methods such as \grpo{} optimize $\pi_\theta$ according to unbiased gradients of the concave objective $J(\pi, r)$ in Equation~\eqref{eq:soft_reward_objective}. In the tabular setting (where $\pi_\theta(a) \propto \exp(\theta_a)$), a small enough decaying learning rate would guarantee convergence to the global optimum. In combinatorial domains, such guarantees require a minuscule learning rate. Indeed, as can be observed empirically in Figure~\ref{fig:therapeutics_data_commons_best_seen_reward}, \grpo{} converges prematurely for acceptable learning rates. In contrast, \name{} captures epistemic uncertainty and thus enables accelerated (sample-efficient) learning without premature convergence.

\subsection{Black-box optimization on bin packing}
\label{sec:bin_packing}

To validate \es{}~\citep{menet2026thompson} as a faithful FunSearch-style baseline~\citep{romera2024mathematical} for program discovery, we reproduce the bin packing task on Weibull distributions from \cite{romera2024mathematical}. By treating our 25 random seeds as independent evolutions (``islands''), \es{} manages to replicate the performance of FunSearch with $10\times$ fewer evaluations ($10^5$ versus $10^6$). Note that our experiments utilize the Qwen3-8B model, whereas FunSearch relied on the proprietary Codey (PaLM 2) model, which is estimated to be orders of magnitude larger.

\begin{table}[h]
    \centering
    \caption{Average excess bins of best bin packing heuristic generated. Heuristics are optimized with feedback on the Weibull 5K distribution. The Weibull 10K distribution is an out-of-distribution test.}
    \small
    \label{tab:bin_packing}
    \begin{tabular}{llll}
        \toprule
        \textbf{Method} & \textbf{\# Evaluations} & \textbf{Weibull 5K} & \textbf{Weibull 10K} \\
        \midrule
        First Fit & N/A & 4.23\% & 4.20\% \\
        Best Fit & N/A & 3.98\% & 3.90\% \\
        FunSearch & $1\times 10^6$ & 0.68\% & 0.32\% \\
        \midrule
        \es{} (1 island) & $4 \times 10^3$ & 0.85\% $\pm$ 0.03\% & 0.61\% $\pm$ 0.03\% \\
        \es{} (25 islands) & $1\times 10^5$ & 0.64\% & 0.41\% \\
        \cdashlinelr{0-3}
        \grpo{} (1 island) & $4 \times 10^3$ & 1.05\% $\pm$ 0.05\% & 1.01\% $\pm$ 0.19\% \\
        \grpo{} (25 islands) & $1\times 10^5$ & 0.76\% & 0.56\% \\
        \cdashlinelr{0-3}
        \name{} (1 island) & $4 \times 10^3$ & 1.02\% $\pm$ 0.03\% & 0.71\% $\pm$ 0.03\% \\
        \name{} (25 islands) & $1\times 10^5$ & 0.83\% & 0.44\% \\
        \bottomrule
    \end{tabular}
\end{table}

While \name{} does not establish a new state-of-the-art for finding bin-packing heuristics---trailing slightly behind the specialized \es{}---these experiments validate a core capability demonstrated also in our RLVR experiments. Specifically, through its uncertainty-aware optimization procedure, \name{} achieves superior out-of-distribution (OOD) transfer compared to standard \grpo{}. Following the original FunSearch protocol, all methods optimize heuristics exclusively on the Weibull 5K setting, with the Weibull 10K setting serving as the unseen OOD evaluation. 

Ultimately, while evolutionary methods excel at specialized local refinement in structured programmatic spaces, \name{} provides scalable, uncertainty-guided global search. Combining these complementary strengths presents an exciting direction for future hybrid black-box optimizers.

\subsection{Extended baselines for black-box optimization}\label{sec:classical_baselines_black_box_optimization}

To isolate the performance gains of LLM-based optimization and provide a fully transparent picture of the black-box optimization landscape, we evaluate domain-specific classical evolutionary search methods on the Quantum Circuit Design and Protein Search tasks. Specifically, we implement a standard Genetic Algorithm (GA) and a $(1+16)$ Evolution Strategy (ES). In all cases, the substitution and insertion/deletion rates were exhaustively tuned via grid-search for optimal performance.

\paragraph{Quantum Circuit Design}
Qiskit code is parameterized as a sequence of quantum gates. Mutation and crossover are directly implemented at the gate granularity, which ensures that no syntax errors can be generated. Both classical methods initialize from an empty circuit.
\begin{itemize}
    \item \textbf{Genetic Algorithm:} Population = 128, Tournament Size = 3, Substitution Rate / Position = 0.05, Insertion \& Deletion Rate / Position = 0.1.
    \item \textbf{$(1+16)$ Evolution Strategy:} Population = 128, Substitution Rate / Position = 0.1, Insertion \& Deletion Rate / Position = 0.2.
\end{itemize}

As shown in Table~\ref{tab:classical_quantum}, the classical evolutionary baselines perform extremely well on this task, eventually surpassing the LLM-based methods at the 8k evaluation mark. However, it is important to note that classical evolutionary search requires hand-engineered mutation and cross-over operations tailored to the environment. In contrast, LLM-based EvoSearch, \grpo{}, and \name{} remain completely task-agnostic and benefit from the strong priors for improved few-step sample efficiency.

\begin{table}[h]
    \centering
    \caption{Comparison against classical non-LLM baselines on Quantum Circuit Design. All methods except GA and ES use a Qwen3-8B model. We report the mean best-seen reward (BSR) on 25 seeds.}
    \label{tab:classical_quantum}
    \resizebox{\linewidth}{!}{
    \begin{tabular}{lccccccc}
        \toprule
        \textbf{Method} & \textbf{Time / Batch} & \textbf{BSR @ 250} & \textbf{BSR @ 500} & \textbf{BSR @ 1k} & \textbf{BSR @ 2k} & \textbf{BSR @ 4k} & \textbf{BSR @ 8k} \\
        \midrule
        \fibo{} & 31.3 s & 10.26 & 10.29 & 10.31 & 10.20 & 10.36 & 10.39 \\
        \tosfit{} & 13.7 s & 10.34 & 10.50 & 10.81 & 11.04 & 11.27 & 11.58 \\
        \es{} & 19.4 s & \textbf{10.86} & \textbf{10.95} & \textbf{11.00} & 11.07 & 11.17 & 11.37 \\
        \grpo{} & 7.1 s & 10.11 & 10.11 & 10.11 & 10.16 & 10.16 & 10.31 \\
        \name{} (T=4) & 9.9 s & 10.31 & 10.45 & 10.77 & 11.15 & 11.39 & 11.52 \\
        \midrule
        Genetic Algorithm & 0.0 s & 9.56 & 9.89 & 10.32 & 11.12 & 12.42 & 12.92 \\
        $(1+16)$ Evolution Strategy & 0.0 s & 9.57 & 10.10 & 10.65 & \textbf{11.54} & \textbf{12.58} & \textbf{12.93} \\
        \bottomrule
    \end{tabular}
    }
\end{table}

\paragraph{Protein Search}
Proteins are parameterized as a sequence of amino acids. Mutation and crossover are implemented at the amino acid granularity, meaning invalid sequences cannot be generated. The initial candidate for the classical methods is a fixed, unconditional sample from ProtGPT2 to match the starting conditions of the LLM-based methods.
\begin{itemize}
    \item \textbf{Genetic Algorithm:} Population = 128, Tournament Size = 3, Substitution Rate / Position = 0.00125, Insertion \& Deletion Rate / Position = 0.008.
    \item \textbf{$(1+16)$ Evolution Strategy:} Population = 128, Substitution Rate / Position = 0.0025, Insertion \& Deletion Rate / Position = 0.004.
\end{itemize}

As shown in Table~\ref{tab:classical_protein}, the classical methods perform significantly worse on the Protein Search task compared to \name{} and (LLM-based) \es{}. While \es{} paired with a massive 284B frontier model (DeepSeek-V4-Flash) eventually reaches a higher asymptote, \name{} applied to a tiny but domain-specific 0.7B model achieves vastly superior sample efficiency early in training while being significantly more computationally efficient. We attribute this to the strong biology priors learnt from pre-training. Note also that evolutionary methods are pure search algorithms; they cannot post-train an LLM. Our main point of comparison therefore remains \name{} versus \grpo{}, as these methods actually update the model's parameters and serve as full reinforcement learning algorithms.

\begin{table}[h]
    \centering
    \caption{Comparison against classical non-LLM baselines and scaled LLM baselines on the Protein Search task. Values denote the mean best-seen reward (BSR) $\pm$ standard error across 25 seeds.}
    \label{tab:classical_protein}
    \resizebox{\linewidth}{!}{
    \begin{tabular}{llccccccc}
        \toprule
        \textbf{Method} & \textbf{Model} & \textbf{Params} & \textbf{Hardware} & \textbf{Time} & \textbf{BSR @ 500} & \textbf{BSR @ 1k} & \textbf{BSR @ 2k} & \textbf{BSR @ 4k} \\
        \midrule
        \name{} & ProtGPT2 & 0.7B & 1x H200 & 7 s & $38.5 \pm 2.7$ & $67.5 \pm 4.5$ & $78.5 \pm 4.9$ & $82.3 \pm 4.9$ \\
        \name{} (T=4) & ProtGPT2 & 0.7B & 1x H200 & 8 s & $\mathbf{70.4 \pm 4.6}$ & $\mathbf{80.4 \pm 5.1}$ & $\mathbf{86.0 \pm 4.9}$ & $92.1 \pm 4.8$ \\
        \cdashlinelr{0-8}
        EvoSearch & Qwen3-8B & 8B & 1x H200 & 15 s & $-3.9 \pm 5.2$ & $14.1 \pm 5.0$ & $20.7 \pm 4.4$ & $26.9 \pm 4.0$ \\
        EvoSearch & Deepseek-V4-Flash & 284B & 2x H200 & 125 s & $7.5 \pm 2.8$ & $35.8 \pm 4.6$ & $68.8 \pm 3.6$ & $\mathbf{95.8 \pm 3.2}$ \\
        \midrule
        Genetic Algorithm & None & 0 & CPU & 0 s & $-10.6 \pm 0.3$ & $1.4 \pm 0.3$ & $16.4 \pm 0.4$ & $35.3 \pm 0.4$ \\
        $(1+16)$ Evo Strategy & None & 0 & CPU & 0 s & $7.6 \pm 0.5$ & $25.4 \pm 0.4$ & $47.7 \pm 0.4$ & $72.3 \pm 0.7$ \\
        \bottomrule
    \end{tabular}
    }
\end{table}

\newpage
\section{Implementation}\label{sec:implementation}

To scale \name{} to billion-parameter LLMs, we implement the Trunk \& Branch architecture using PyTorch's native automatic differentiation and dynamic execution graph. The codebase is designed to avoid $n$ redundant passes through the shared trunk during both training and generation.

\paragraph{Single forward/backward pass training for the full ensemble} 
A naive ensemble requires $n$ independent forward \& backward passes. In \name{}, we execute a single forward pass through the shared trunk. The resulting intermediate hidden states are then branched into the $n$ LoRA heads, which are computed in parallel via batched matrix multiplication. Because all $n$ parallel heads branch off from the same intermediate node, calling \texttt{loss.backward()} on the summed (Poisson bootstrapped) ensemble loss natively forces PyTorch's autograd engine to combine and accumulate the gradients from all branches. This accumulated gradient is then propagated through the shared trunk in a single backward pass, decoupling the trunk's computational cost from the ensemble size $n$.
\paragraph{Batched Thompson sampling generation} 
During inference, Thompson sampling requires, for each sample, drawing a policy from the approximate posterior, which we implement by selecting one of the $n$ ensemble members. To maximize generation throughput, we sample multiple candidate solutions in a batch, where each sequence is generated by a randomly drawn LoRA branch. To achieve this without instantiating $n$ separate model copies, we dynamically patch the forward pass of the base language modeling head. As shown in Listing \ref{lst:batched_gen}, we route the hidden states of each sequence in the batch to its assigned LoRA weights ($A_i$, $B_i$), allowing the entire batch to be processed concurrently.
\begin{lstlisting}[language=Python, float=htbp, basicstyle=\ttfamily\small, numbers=left, numberstyle=\tiny, frame=lines, columns=fullflexible, keepspaces=true, keywordstyle=\color{blue!60!black}\bfseries, commentstyle=\color{Green}\itshape, stringstyle=\color{Red}, caption={Dynamic routing for batched Thompson sampling generation.}, label=lst:batched_gen]
# 1. Sample an active head index for each sequence in the batch
active_heads = torch.randint(0, n_ensembles, (batch_size,))
# 2. Gather specific LoRA weights for the sampled heads
sel_A = torch.stack([ensembles_A[i] for i in active_heads])
sel_B = torch.stack([ensembles_B[i] for i in active_heads])
# 3. Patch the forward pass of the final linear layer
original_forward = base_projection.forward
def batched_lora_forward(hidden_states):
    # hidden_states shape: (batch_size, seq_len, hidden_dim)
    base_logits = torch.nn.functional.linear(hidden_states, base_projection.weight)
    lora_logits = torch.bmm(
        torch.bmm(hidden_states, sel_A.transpose(1, 2)), sel_B.transpose(1, 2)
    )
    return base_logits + lora_logits
base_projection.forward = batched_lora_forward
# 4. Generate the batch concurrently using the patched head
generated_ids = generate_tokens(prompt_ids, model, group_size)
base_projection.forward = original_forward
\end{lstlisting}

\newpage
\section{Experimental details}\label{sec:experimental_details}
To facilitate reproducibility, we describe the exact setup for each of the experiments presented in Section~\ref{sec:experiments}. Additionally, our complete code base is available at \url{https://github.com/IBM/thompson-sampling-via-fine-tuning-of-llms}, including extensive experimental configuration files and an environment file specifying all Python packages and their versions. Our implementation is based on Pytorch~\citep{paszke2019pytorchimperativestylehighperformance} and Huggingface Transformers~\citep{wolf2020huggingfacestransformersstateoftheartnatural}.\looseness-1

To ensure statistical significance, we run all experiments for 25 seeds and report on the mean and standard error. Each experiment was conducted on a single NVIDIA H200 GPU (141GB). The total cost of running all experiments, including preliminary hyper-parameter tuning that did not make it into the paper, amounts to $145\ 000$ GPU hours. More detailed information is provided in Table~\ref{tab:experiment_run_times}.

\begin{table}[h]
    \centering
    \caption{Runtime of a single seed of \name{} on each of the main reported benchmarks.}
    \label{tab:experiment_run_times}
    \begin{tabular}{cccccc}
        \toprule
         FAQ & Protein & Quantum Circuit & TDC (all) & AIME & LiveCodeBench v6\\
         \midrule
         2h 0m & 50m & 8h 40m & 5h 0m & 16h 30m & 3d 23h 40m\\
         \bottomrule
    \end{tabular}
\end{table}

Because the large memory footprint of the logits would make memory bandwidth a bottleneck, we chunk the readout layer along the sequence dimension to maximize cache efficiency. By processing smaller blocks, specifically, sizes of 256 for black-box optimization and 512 for reinforcement learning, we keep the data within the cache and drastically reduce slow memory transfers. The strictness of this bandwidth limitation is evident in our quantum circuit design experiments, where using larger chunks of 512 resulted in a slight slowdown despite fitting into VRAM.

\subsection{\name{} for black-box optimization (excluding therapeutics data commons benchmarks)}\label{sec:experimental_details_pure_exploration}
Across all tasks and methods, batched black-box optimization is conducted by generating solutions in groups of size $G=16$. We consistently adopt an ensemble size of $n=16$, with a LoRA rank of $r=128$ for the branches. The LoRA $B$ matrices are initialized to zero, meaning the ensemble members initially behave identically to the shared backbone before diversifying during training via Poisson bootstrapping. \grpo{} is given one LoRA branch to fine-tune separately from the backbone.

To implement \name{}, we bootstrap the data given to the \grpo{}~\citep{shao2024deepseekmath} objective (more precisely the variant proposed by \cite{zheng2025groupsequencepolicyoptimization}) with a clipping parameter of $\epsilon=0.2$. For \name{}, we add the KL-regularization terms to the computation of the baseline to ensure consistency with Proposition~\ref{prop:poets_gradients}. The shared backbone and the ensemble branches are optimized using the AdamW optimizer with $\beta_1=0.9$ and $\beta_2=0.999$ and no weight decay. Given that generation in these tasks requires broad exploration, the generation temperature is uniformly set to $1.0$ across all methods. We retain the same learning rates independently of the size of the training buffer.

The learning rate of both the backbone and the ensembles is always tuned for \grpo{} and carried over zero-shot to \name{} except for the FAQ response refinement task, where we double the ensemble learning rate to ensure sufficient ensemble diversity. The learning rate of the last-layer LoRA branches can typically be set much larger than that of the backbone while retaining stability. An additional benefit of this decomposition into fast-weights and slow-weights \citep{kahneman2011thinking} is that the shared backbone cannot easily bypass the diversity introduced from bootstrapping the ensemble.

For each specific task, the problem setting dictates the prompt template, system instructions, and precise hyperparameter tuning:

\paragraph{FAQ refinement} We optimize a Qwen3-8B model with the system prompt set to \texttt{"You are a helpful assistant."} to generate responses up to 512 tokens. Thinking mode is disabled and the prompt given to the model for generation is: 
\begin{verbatim}    Write an FAQ response to the question "How do I reset my password?".
\end{verbatim}
The reward function consists of the text similarity (cosine similarity score) between the generated response and a hidden ground-truth target text, computed using the \texttt{Qwen3-Embedding-0.6B} model.

As a reward function, we consider the alignment to the unknown ground-truth response \textit{To reset your password, go to the login page and click the “Forgot Password” link. Enter your registered email address, and we'll send you instructions to create a new password. Make sure to check your spam folder if you don’t see the email within a few minutes.}, computed as the cosine similarity between the Qwen3-Embedding-0.6B~\citep{qwen3embedding} embeddings of the candidate and the ground-truth.

To avoid premature entropy collapse of the \grpo{} baseline, we set the KL-divergence coefficient for all gradient-based methods to $\beta=1.0 \times 10^{-4}$ against the base reference model set to the LLM at round $0$, with no additional entropy regularization ($\alpha=0.0$). The learning rates are $1.0\times 10^{-6}$ for the shared backbone and $2.0 \times 10^{-3}/1.0\times 10^{-3}$ for the ensemble branches (\name{}/\grpo{}).

\paragraph{Protein search} We optimize the ProtGPT2 model (0.738B parameters), starting with an empty generation prompt (\textit{<|endoftext|>}) to freely sample amino acid sequences up to 512 tokens. The objective is to maximize thermal stability, estimated via the negative thermal instability index \citep{guruprasad1990correlation}. Once again, we apply a KL-penalty of $\beta=1.0\times 10^{-4}$ without entropy regularization ($\alpha=0.0$). Learning rates are elevated for this smaller architecture: $1.0 \times 10^{-4}$ for the shared backbone and $1.0 \times 10^{-2}$ for the ensemble branches.

\begin{table}[b!]
\caption{Hyperparameters for black-box optimization. The learning rates are tuned for \grpo{} and then directly transferred to \name{} except for FAQ Refinement, where the LoRA LR was doubled to increase diversification. We keep a fixed learning rate for all replay buffer sizes. For FAQ refinement, we swept the backbone LR across $1\!\times\!10^{-6}$, $2  \!\times\! 10^{-6}$, $5 \!\times\! 10^{-6}$, $1 \!\times\! 10^{-5}$, $2 \!\times\! 10^{-5}$ and the LoRA LR across $5 \!\times\! 10^{-4}$, $1 \!\times\! 10^{-3}$, $2 \!\times\! 10^{-3}$. For protein search, we swept the backbone LR across $5 \!\times\! 10^{-5}, 1 \!\times\! 10^{-4}, 2 \!\times\! 10^{-4}$ and the LoRA LR across $5 \!\times\! 10^{-3}$, $1 \!\times\! 10^{-2}$, $2 \!\times\! 10^{-2}$. For quantum circuit design, we swept the backbone learning rate across $5 \!\times\! 10^{-7}, 1 \!\times\! 10^{-6}, 2 \!\times\! 10^{-6}$ and the LoRA LR across $2 \!\times\! 10^{-4}, 5 \!\times\! 10^{-4}, 1 \!\times\! 10^{-3}, 2 \!\times\! 10^{-3}$. Appendix~\ref{sec:additional_results} ablates other hyper-parameters.}
\renewcommand{\arraystretch}{1.2}
\small\centering
\begin{tabular}{lcccccc}
    \toprule
    \multirow{2}{*}{Hyperparameter} & \multicolumn{2}{c}{FAQ Refinement} & \multicolumn{2}{c}{Protein Search} & \multicolumn{2}{c}{Quantum Circuit Design} \\
    \cmidrule(lr){2-3} \cmidrule(lr){4-5} \cmidrule(lr){6-7}
    & \name{} & \grpo{} & \name{} & \grpo{} & \name{} & \grpo{} \\
    \cmidrule(lr) {1-1} \cmidrule(lr){2-3} \cmidrule(lr){4-5} \cmidrule(lr){6-7}
    Group Size & 16 & 16 & 16 & 16 & 16 & 16 \\
    \grpo{} Variant & GSPO & GSPO & GSPO & GSPO & GSPO & GSPO \\
    Importance Clipping & 0.2 & 0.2 & 0.2 & 0.2 & 0.2 & 0.2 \\
    \cmidrule(lr) {1-1} \cmidrule(lr){2-3} \cmidrule(lr){4-5} \cmidrule(lr){6-7}
    Backbone LR & $10^{-6}$ & $10^{-6}$ & $ 10^{-4}$ & $ 10^{-4}$ & $ 10^{-6}$ & $ 10^{-6}$ \\
    LoRA LR & $2 \!\times\! 10^{-3}$ & $ 10^{-3}$ & $ 10^{-2}$ & $ 10^{-2}$ & $5.0 \!\times\! 10^{-4}$ & $5.0 \!\times\! 10^{-4}$ \\
    AdamW $\beta_1$ & 0.9 & 0.9 & 0.9 & 0.9 & 0.9 & 0.9 \\
    AdamW $\beta_2$ & 0.999 & 0.999 & 0.999 & 0.999 & 0.999 & 0.999\\ 
    \cmidrule(lr) {1-1} \cmidrule(lr){2-3} \cmidrule(lr){4-5} \cmidrule(lr){6-7}
    Weight Decay & 0 & 0 & 0 & 0 & 0 & 0 \\
    KL-Regularization $\beta$ & $10^{-4}$ & $10^{-4}$ & $10^{-4}$ & $10^{-4}$ & $10^{-3}$ & $10^{-3}$ \\
    Entropy Regularization $\alpha$ & 0 & 0 & 0 & 0 & 0.02 & 0.02 \\
    \cmidrule(lr) {1-1} \cmidrule(lr){2-3} \cmidrule(lr){4-5} \cmidrule(lr){6-7}
    Poisson Bootstrapping & True & False & True & False & True & False\\
    Ensemble Size & 16 & 1 & 16 & 1 & 16 & 1 \\
    LoRA Rank & 128 & 128 & 128 & 128 & 128 & 128 \\
    \cmidrule(lr) {1-1} \cmidrule(lr){2-3} \cmidrule(lr){4-5} \cmidrule(lr){6-7}
    Temperature & 1.0 & 1.0 & 1.0 & 1.0 & 1.0 & 1.0 \\
    Max Output Tokens & 512 & 512 & 512 & 512 & 512 & 512 \\
    \bottomrule
\end{tabular}
\label{tab:sweep_hyperparameters_blackbox}
\end{table}

\paragraph{Quantum circuit design} We use a Qwen3-8B model with system prompt set to \texttt{"You are a helpful assistant."} to generate quantum circuits with up to 512 tokens of Qiskit~\citep{vishwakarma2024qiskit}. Thinking is disabled and the model is presented with the following task prompt: \begin{verbatim}    Define a 7-qubit Qiskit quantum circuit that can be called with qc =
    quantum_circuit(). Include all imports and enclose your code block with
    triple backticks in markdown format (```).
\end{verbatim}
The reward corresponds to the simulated (negative) energy of the state prepared by the generated circuit under a complex, highly entangled synthetic Pauli Hamiltonian. If the candidate code contains invalid syntax or fails to run, it receives a harsh minimal reward. Because the model has very low initial entropy, small weight updates can easily introduce entropy collapse. Thus, we apply an explicit entropy regularization penalty of $\alpha=0.02$, alongside a KL-penalty of $\beta=1.0\times 10^{-3}$ across all gradient-based methods. The learning rates are $1.0\times 10^{-6}$ for the shared backbone and $5.0 \times 10^{-4}$ for the ensemble branches.

\subsubsection{Baseline implementations for black-box optimization (excluding TDC)}

\textbf{\fibo{}}~\citep{de2025simplifying, menet2026codingagentsenvironmentinteraction} implements Thompson sampling fully in context by attending to a history of best-scoring previously evaluated candidate solutions and their associated rewards in a chat format. The memory size of the history is set to $128$ for FAQ Refinement and Protein Search, but reduced to $64$ for Quantum Circuit Design to prevent out-of-memory errors. Thinking is disabled and in-context Bayesian optimization is induced via the following system prompt: 
\begin{verbatim}
    You are conducting Bayesian optimization fully in context. You are 
    provided a history of candidates with associated rewards in a chat 
    format. Propose a new candidate that might maximize the reward. Your
    novel solution must be enclosed using markdown (...). Never repeat a
    previous solution. Your search space is over [TASK SPECIFIC SPACE]. \end{verbatim} 
The history is initially seeded with an example response (for FAQ and Protein Search) or an empty quantum circuit (for Quantum Circuit Design).

\textbf{\es{}}~\citep{romera2024mathematical} performs evolutionary search in context through instruction following. It maintains a population of size $128$ and selects two parent candidates in each round via tournament selection with a tournament size of $3$. The language model acts as the crossover and mutation operator prompted by: 
\begin{verbatim}
    You are conducting evolutionary search in context. You are provided two
    candidates. Propose a new distinct candidate by combining the two
    provided candidates and mutating the result. Your novel candidate must 
    be enclosed using markdown (...). Never repeat previous candidates. 
    Your search space is over [TASK SPECIFIC SPACE].
\end{verbatim}
Chain-of-thought reasoning is disabled. Similar to \fibo{}, the initial population is seeded with an example response or an empty circuit to provide the desired format to the model.

\textbf{\tosfit{}}~\citep{menet2026thompson} fine-tunes the language model toward a variational Thompson sampling policy based on a Gaussian process reward surrogate. Across all experiments, an exploration bonus of $4.0$ is applied to the prior amplitude inferred via marginal likelihood maximization, and the noise-to-amplitude ratio is set to $0.01$. The marginal likelihood maximization uses $4$ warmup steps. For the Gaussian process feature map, FAQ Refinement uses the primary 256 dimensions of the \texttt{Qwen3-Embedding-0.6B}~\citep{qwen3embedding} model, Protein Search uses the mean token embeddings from ProtGPT2 projected to the unit sphere (1280 dimensions), and Quantum Circuit Design uses 211 two-qubit Pauli observables. Following \grpo{} and \name{}, we apply a KL-penalty of $\beta=1.0\times 10^{-4}$ for FAQ and Protein Search, and $\beta=1.0\times 10^{-3}$ for Quantum Circuit Design. The learning rates are tuned individually, reaching $1.0\times 10^{-6}$, $1.0\times 10^{-4}$, and $2.0 \times 10^{-6}$, respectively.

\subsection{\name{} on therapeutics data commons (TDC) molecular design benchmarks}
\label{sec:appendix_tdc_settings}

We utilize the \texttt{GP-MoLFormer-Uniq}~\citep{ross2025gp} molecule generator model (47M parameters pre-trained autoregressively on 1.1 billion SMILES strings) to unconditionally sample SMILES strings up to 200 tokens. The objective in our TDC benchmarks is to navigate highly non-linear, multi-objective search landscapes across three distinct biological profiles modeled by ExCAPE-DB trained wetlab surrogates: (i) \textbf{GSK3$\beta$}, a kinase regulating glycogen metabolism and a major target for Alzheimer's and diabetes (Random Forest oracle); (ii) \textbf{JNK3}, a MAPK mediating stress-induced neuronal apoptosis targeted for neuroprotective therapies (Random Forest oracle); and (iii) \textbf{DRD2}, a GPCR targeted by antipsychotic medications for schizophrenia and bipolar disorder (SVM oracle). 

For each target, the agent optimizes a composite multiplicative reward that balances target-specific binding probability against drug-likeness and synthesizability: 
\begin{equation}
    \text{Reward} = \text{Score}_{\text{Target}} \times \text{QED} \times \text{SA}_{\text{Norm}}
\end{equation}
Here, $\text{Score}_{\text{Target}} \in [0, 1]$ is the oracle binding probability, $\text{QED} \in [0, 1]$ assesses drug-likeness (e.g., molecular weight, $\log P$) via RDKit~\citep{rdkit} and is essential to gauge bioavailability, and $\text{SA}_{\text{Norm}} \in [0, 1]$ normalizes the Synthetic Accessibility score~\citep{ertl2009estimation} via $\frac{10.0 - \max(1.0, \min(10.0, \text{SA}))}{9.0}$, measuring the difficulty of synthesizing the drug in a lab or in production.

Across all therapeutic targets, batched black-box optimization is conducted by generating SMILES strings in groups of size $G=16$. We consistently adopt an ensemble size of $n=64$, with a LoRA rank of $r=768$ for the branches (equal to the hidden layer size of the generator model to provide full-rank adaptation capability). Noticing that both \name{} and vanilla \grpo{} achieve better performance by only fine-tuning the final readout layer, we freeze the (shared) model trunk. As in all other tasks, the LoRA $B$ matrices are initialized to zero, meaning the ensemble members initially behave identically to the shared backbone before diversifying during training via Poisson bootstrapping. \grpo{} is given one ``LoRA'' branch to fine-tune separately from the frozen backbone. 

\grpo{} (including \name{}) is stabilized with a clipping parameter of $\epsilon=0.2$ and a KL-divergence penalty of $\beta=0.1$ against the base reference model without additional entropy regularization ($\alpha=0.0$). We adopt a replay buffer containing the last $16$ batches and enable importance sampling in the style of GSPO~\citep{zheng2025groupsequencepolicyoptimization}. The active LoRA branches are optimized using the AdamW optimizer with $\beta_1=0.9$, $\beta_2=0.999$, and no weight decay. The generation temperature is uniformly set to $1.0$. Optimization runs for 8000 interactions/generations with a batch size of $16$.

\begin{table}[h]
\caption{Hyperparameters for TDC Molecular Design. The shared trunk is completely frozen, and only the full-rank last-layer adapter is optimized. On these tasks, with \texttt{GP-MoLFormer-Uniq} as generator, updating the backbone model parameters deteriorates performance of \grpo{} and \name{}.\looseness -1}
\small\centering
\renewcommand{\arraystretch}{1.2}
\begin{tabular}{lcc}
    \toprule
    \multirow{2}{*}{Hyperparameter} & \multicolumn{2}{c}{TDC Targets (All)} \\
    \cmidrule(lr){2-3} 
    & \name{} & \grpo{} \\
    \cmidrule(lr) {1-1} \cmidrule(lr){2-3} 
    Group Size & 16 & 16 \\
    \grpo{} Variant & GSPO & GSPO \\
    Importance Clipping & 0.2 & 0.2 \\
    \cmidrule(lr) {1-1} \cmidrule(lr){2-3} 
    Backbone LR & 0.0 & 0.0 \\
    LoRA LR & ablated & ablated \\
    AdamW $\beta_1$ & 0.9 & 0.9 \\
    AdamW $\beta_2$ & 0.999 & 0.999\\ 
    \cmidrule(lr) {1-1} \cmidrule(lr){2-3} 
    Weight Decay & 0 & 0 \\
    KL-Regularization $\beta$ & 0.1 & 0.1 \\
    Entropy Regularization $\alpha$ & 0 & 0 \\
    \cmidrule(lr) {1-1} \cmidrule(lr){2-3} 
    Poisson Bootstrapping & True & False\\
    Ensemble Size & 64 & 1 \\
    LoRA Rank & 768 & 768 \\
    \cmidrule(lr) {1-1} \cmidrule(lr){2-3} 
    Temperature & 1.0 & 1.0 \\
    Max Output Tokens & 200 & 200 \\
    \bottomrule
\end{tabular}
\label{tab:sweep_hyperparameters_tdc}
\end{table}

\subsection{\name{} for bin packing}
\label{sec:appendix_bin_packing_settings}

We evaluate \name{} on the online 1D bin packing problem following the programmatic heuristic discovery setup introduced by FunSearch~\citep{romera2024mathematical}. In online bin packing, a stream of items of varying sizes arrives sequentially and must be immediately assigned to one of an unlimited set of bins with capacity $C = 100$ without knowledge of future items. The objective is to discover a Python function named \texttt{heuristic(item\_size: float, bin\_capacities: np.ndarray) -> np.ndarray} that computes priority scores for all currently feasible bins (bins with remaining capacity at least equal to \texttt{item\_size}), with the item placed in the bin receiving the maximum score.

\paragraph{Dataset and evaluation} Item sizes are drawn from a Weibull distribution: item sizes are generated as $45 \times \mathrm{Weibull}(3)$, rounded to the nearest integer, and clipped to the range $[1, 100]$~\citep{romera2024mathematical}. Following FunSearch, heuristics are optimized against feedback on instances with $5{,}000$ items (\textbf{Weibull 5K}), while generalization is evaluated out-of-distribution (OOD) on instances with $10{,}000$ items (\textbf{Weibull 10K}). Each evaluation averages over $5$ simulated packing experiments. The performance metric is the \textit{average excess bins} relative to the theoretical $L_2$ lower bound~\citep{romera2024mathematical}:
\begin{equation}
    \text{Excess Bins} = \frac{\text{Bins Used}}{\text{Bins}_{L_2}} - 1.0.
\end{equation}
The environment reward is the negative excess bins, $r = -\text{Excess Bins}$ (higher is better, with an ideal optimum of $0.0$). Unexecutable scripts, syntax errors, or functions exceeding execution timeouts in the isolated sandbox return \texttt{NaN} and receive an invalid generation penalty.

\paragraph{Prompting and generation} We optimize the \texttt{Qwen3-8B} model~\citep{yang2025qwen3technicalreport}. Thinking mode is disabled, and the generation limit is set to $512$ new tokens with sampling temperature $1.0$. The system prompt is:
\begin{verbatim}
You are an expert algorithm designer tasked with discovering novel, 
innovative, and non-obvious online bin packing heuristics. Avoid 
plain textbook heuristics like standard Best-Fit, First-Fit, or 
Worst-Fit. Design a creative mathematical priority score function.
/no_think
\end{verbatim}
The user prompt provided to the model is:
\begin{verbatim}
Define a Python function named `heuristic(item_size: float, 
bin_capacities: np.ndarray) -> np.ndarray` that assigns novel 
and creative priority scores to bin capacities for online bin 
packing to minimize excess bins. Avoid outputting standard 
textbook heuristics like plain Best-Fit or First-Fit. Output 
only the function definition and body inside triple backticks 
(```python ... ```)/no_think.
\end{verbatim}

\paragraph{Optimization and hyperparameters} Optimization runs for $4{,}000$ interactions in groups of size $G=32$, corresponding to $125$ policy gradient updates. We operate in the online setting without a replay buffer ($T=1$). Both \grpo{} and \name{} optimize the GSPO loss~\citep{zheng2025groupsequencepolicyoptimization} with clipping parameter $\epsilon=0.2$, group advantage standardization, importance sampling, and length rescaling. We apply a KL-divergence penalty of $\beta = 1.0 \times 10^{-3}$ against the base reference model without entropy regularization ($\alpha = 0.0$). Validation on Weibull 10K is performed every $4$ rounds.

For \name{}, we maintain an ensemble of $n=16$ policy heads with LoRA rank $r=128$, initialized with $B_i = 0$. Epistemic uncertainty is captured via online Poisson bootstrapping with weights $w \sim \mathrm{Poisson}(1)$. For this task, the shared backbone language head is frozen (\texttt{freeze\_backbone\_head: True}), confining policy differentiation entirely to the LoRA branches. For \grpo{}, we fine-tune the backbone head directly ($n=1$) without LoRA heads and without bootstrapping. The backbone learning rate was swept across $1.0\times 10^{-6}$, $2.0\times 10^{-6}$, and $4.0\times 10^{-6}$ for \grpo{}; the optimal rate ($2.0\times 10^{-6}$) was transferred to \name{}. For \name{}, the LoRA ensemble learning rate was swept across $2.0\times 10^{-5}$, $4.0\times 10^{-5}$, and $8.0\times 10^{-5}$, with $4.0\times 10^{-5}$ achieving the best performance. Both components are trained using AdamW with $\beta_1=0.9$, $\beta_2=0.999$, and zero weight decay. Table~\ref{tab:sweep_hyperparameters_bin_packing} summarizes all hyperparameters.

\paragraph{\es{} baseline setup} For the \es{} baseline~\citep{romera2024mathematical, menet2026thompson}, we maintain an evolutionary population of $128$ candidate heuristics with a tournament selection size of $2$. The population is seeded with an initial candidate implementing First-Fit/Worst-Fit: \texttt{def heuristic(item\_size, bin\_capacities): return -bin\_capacities}. The LLM acts as the crossover and mutation operator prompted to combine two parent heuristics into an improved child heuristic. Following the multi-island design of FunSearch, our 25 seeds act as 25 independent evolutionary islands.

\begin{table}[h]
\caption{Hyperparameters for Online Bin Packing Heuristic Search. The backbone learning rate was optimized on \grpo{} and transferred to \name{}. In \name{}, the backbone language head is frozen and only the LoRA adapter branches are updated.}
\renewcommand{\arraystretch}{1.2}
\small\centering
\begin{tabular}{lcc}
    \toprule
    \multirow{2}{*}{Hyperparameter} & \multicolumn{2}{c}{Bin Packing (Qwen3-8B)} \\
    \cmidrule(lr){2-3}
    & \name{} & \grpo{}\\
    \cmidrule(lr) {1-1} \cmidrule(lr){2-3}
    Group Size $G$ & 32 & 32 \\
    Number of Interactions & 4000 & 4000 \\
    Replay Buffer Size $T$ & 1 & 1 \\
    \grpo{} Variant & GSPO & GSPO \\
    Importance Clipping $\epsilon$ & 0.2 & 0.2 \\
    \cmidrule(lr) {1-1} \cmidrule(lr){2-3}
    Backbone LR & $2.0 \!\times\! 10^{-6}$ & $2.0 \!\times\! 10^{-6}$ \\
    LoRA LR & $4.0 \!\times\! 10^{-5}$ & $0.0$ \\
    Freeze Backbone Head & True & False \\
    AdamW $\beta_1$ & 0.9 & 0.9 \\
    AdamW $\beta_2$ & 0.999 & 0.999 \\
    Gradient Clip Norm & 1.0 & 1.0 \\
    \cmidrule(lr) {1-1} \cmidrule(lr){2-3}
    Weight Decay & 0 & 0 \\
    KL-Regularization $\beta$ & $10^{-3}$ & $10^{-3}$ \\
    Entropy Regularization $\alpha$ & 0 & 0 \\
    \cmidrule(lr) {1-1} \cmidrule(lr){2-3}
    Poisson Bootstrapping & True & False \\
    Ensemble Size $n$ & 16 & 1 \\
    LoRA Rank $r$ & 128 & 128 \\
    \cmidrule(lr) {1-1} \cmidrule(lr){2-3}
    Temperature & 1.0 & 1.0 \\
    Max Output Tokens & 512 & 512 \\
    \bottomrule
\end{tabular}
\label{tab:sweep_hyperparameters_bin_packing}
\end{table}

\subsection{\name{} for reinforcement learning}\label{sec:experimental_details_reinforcement_learning}
Across all tasks, batched optimization is conducted by sampling $4$ prompts and generating candidate solutions in groups of size $G=8$, i.e., a batch of 32 responses is formed and evaluated. To avoid out-of-memory during training, we vectorize computation across groups but treat each prompt as a separate sequential mini-batch. For all experiments, we adopt an ensemble size of $n=16$, with a LoRA rank of $r=128$ for the branches in \name{}. As before, the LoRA $B$ matrices are initialized to zero. \grpo{} is fine-tuned directly without any additional LoRA branches. To implement \name{}, we bootstrap the \grpo{} objective (with both importance sampling and length rescaling enabled) and retain its clipping parameter of $\epsilon=0.2$. We apply a consistent KL-divergence penalty coefficient of $\beta=1.0\times 10^{-3}$ against the base reference model across all methods and include it in the baseline calculation. We do not use additional entropy regularization ($\alpha=0.0$). The shared backbone and the ensemble branches are optimized using the AdamW optimizer with hyperparameters $\beta_1=0.9$, $\beta_2=0.999$, no weight decay, and a gradient norm clipping with maximal value $1.0$.

\begin{table}[t]
\caption{Hyperparameters for RLVR on AIME. The backbone learning rates are tuned for \grpo{} and then directly transferred to \name{}. Since LoRA is not used for \grpo{}, the LoRA learning rates are tuned for \name{}. More precisely, for AIME 2026 and AIME 1983-2024 (T=1), we swept the backbone LR across $5 \!\times\! 10^{-6}, 1 \!\times\! 10^{-5}, 2 \!\times\! 10^{-5}$ and the LoRA LR across $2  \!\times\! 10^{-4}$, $5 \!\times\! 10^{-4}$, $1  \!\times\! 10^{-3}$. For AIME 1983-2024 (T=4), we swept the backbone LR across $1\!\times\! 10^{-6}$, $2\!\times\! 10^{-6}$, $2.5\!\times\! 10^{-6}$, $3\!\times\! 10^{-6}$, $4\!\times\! 10^{-6}$ and upon noticing that the optimal learning rate is $4$ times smaller than for $T=1$, we also decreased the LoRA LR by a factor $4$. Appendix~\ref{sec:additional_results} ablates other hyper-parameters.}
\small\centering
\renewcommand{\arraystretch}{1.2}
\begin{tabular}{lcccc}
    \toprule
    \multirow{2}{*}{Hyperparameter} & \multicolumn{2}{c}{AIME 2026 (T=1,4) \& AIME 1983-2024 (T=1)} & \multicolumn{2}{c}{AIME 1983-2024 (T=4)} \\
    \cmidrule(lr){2-3} \cmidrule(lr){4-5}
    & \name{} & \grpo{} & \name{} & \grpo{}\\
    \cmidrule(lr) {1-1} \cmidrule(lr){2-3} \cmidrule(lr){4-5}
    Question Batch Size & 4 & 4 & 4 & 4\\
    Group Size & 8 & 8 & 8 & 8 \\
    \grpo{} Variant & GSPO & GSPO & GSPO & GSPO \\
    Importance Clipping & 0.2 & 0.2 & 0.2 & 0.2 \\
    \cmidrule(lr) {1-1} \cmidrule(lr){2-3} \cmidrule(lr){4-5}
    Backbone LR & $10^{-5}$ & $10^{-5}$ & $2.5 \!\times \!10^{-6}$ & $2.5 \!\times \!10^{-6}$\\
    LoRA LR & $5 \!\times\! 10^{-4}$ & N/A & $ 1.25 \!\times\! 10^{-4}$ & N/A\\
    AdamW $\beta_1$ & 0.9 & 0.9 & 0.9 & 0.9\\
    AdamW $\beta_2$ & 0.999 & 0.999 & 0.999 & 0.999\\
    Gradient Clip Norm & 1.0 & 1.0 & 1.0 & 1.0\\
    \cmidrule(lr) {1-1} \cmidrule(lr){2-3} \cmidrule(lr){4-5}
    Weight Decay & 0 & 0 & 0 & 0\\
    KL-regularization $\beta$ & $10^{-3}$ & $10^{-3}$ & $10^{-3}$ & $10^{-3}$\\
    Entropy regularization $\alpha$ & 0 & 0 & 0 & 0\\
    \cmidrule(lr) {1-1} \cmidrule(lr){2-3} \cmidrule(lr){4-5}
    Poisson Bootstrapping & True & False & True & False\\
    Ensemble Size & 16 & 1 & 16 & 1\\
    LoRA Rank & 128 & 0 & 128 & 0\\
    \cmidrule(lr) {1-1} \cmidrule(lr){2-3} \cmidrule(lr){4-5}
    Temperature & 1.0 & 1.0 & 1.0 & 1.0\\
    Max Output Tokens & 2048 & 2048 & 2048 & 2048\\
    \bottomrule
\end{tabular}
\label{tab:sweep_hyperparameters_rl_aime}
\end{table}

For both \grpo{} and \name{}, the backbone learning rate is $1.0\times 10^{-5}$, except for the off-policy variants ($T=4$) trained on AIME1983-2024, where a backbone learning rate of $2.5\times 10^{-6}$ is adopted. In all cases the learning rate was optimized for \grpo{} via grid search and retained for \name{} without further tuning. For \name{}, the learning rate of the ensemble branches is tuned via grid search and set to $5.0 \times 10^{-4}$ except for the off-policy variants ($T=4$) trained on AIME1983-2024, where it is set to $1.25 \times 10^{-4}$, and for LiveCodeBench v6, where we directly ablate on it in our main results. As for black-box optimization, we find that the large learning rate of the LoRA heads is crucial to achieve calibrated policy diversification.

We intentionally use two complementary RLVR experimental designs. For AIME, \grpo{} is run with its standard configuration (no separate LoRA head), representing the strongest published single-policy baseline; \name{} extends this with a Trunk \& Branch architecture. For LiveCodeBench, we give both \grpo{} and \name{} an identical Trunk \& Branch architecture (LoRA rank 128 on the readout layer, same dual LR scheme), explicitly isolating the contribution of Poisson bootstrapping from architectural changes and learning dynamics. The LiveCodeBench experiment thus serves as the controlled ablation, and in fact demonstrates that \grpo{} does not benefit from two separate learning rates.\looseness-1

\paragraph{AIME} We optimize a Qwen3-8B model with temperature $1.0$ and generation limit set to 2048 new tokens. Thinking mode is disabled, and the model is prompted with the following system instructions:
\begin{verbatim}    You are an expert mathematician. You must reason step-by-step and place 
    your final answer within \boxed{}. Be concise.\end{verbatim}

\paragraph{AIME 2026} Models are trained over a total of $4800$ interactions, equating to $150$ gradient updates and $20$ passes through the dataset. In the spirit of classical contextual bandits where there is no train-test split, the model is also validated on AIME 2026.
    
\paragraph{AIME 1983-2024} Models are trained over a total of $8192$ interactions, which corresponds to $256$ gradient updates and $1.1$ passes through the dataset. We report accuracy both on the training set (i.e., the observed rewards) as well as on a separate held-out validation set (AIME2026).

\paragraph{LiveCodeBench v6} We optimize a Qwen3-8B model with generation limit set to $1536$ new tokens. Thinking mode is disabled, and the model is prompted with the following system instructions:

\begin{verbatim}
    You are an expert programmer. You must place your final answer within 
    a markdown python block (```python ... ```). Be concise.
\end{verbatim}

We follow~\citep{guo_2025} and decrease the validation temperature to $0.6$ while retaining the training temperature at $1.0$ for rollout generation. To improve training efficiency and GPU utilization, we train with a replay buffer of size T=4.

\begin{table}[h]
\caption{Hyperparameters for RLVR on LiveCodeBench. The backbone learning rates are tuned for \grpo{} and then directly transferred to \name{}. More precisely, we swept the backbone LR across $5 \!\times\! 10^{-6}, 1 \!\times\! 10^{-5}, 2 \!\times\! 10^{-5}$, with $1 \!\times\! 10^{-5}$ achieving the highest accuracy on the private tests.}
\renewcommand{\arraystretch}{1.2}
\small\centering
\begin{tabular}{lcc}
    \toprule
    \multirow{2}{*}{Hyperparameter} & \multicolumn{2}{c}{LiveCodeBench v6 (T=4)} \\
    \cmidrule(lr){2-3}
    & \name{} & \grpo{}\\
    \cmidrule(lr) {1-1} \cmidrule(lr){2-3}
    Question Batch Size & 4 & 4\\
    Group Size & 8 & 8 \\
    \grpo{} Variant & GSPO & GSPO \\
    Importance Clipping & 0.2 & 0.2 \\
    \cmidrule(lr) {1-1} \cmidrule(lr){2-3}
    Backbone LR & $10^{-5}$ & $10^{-5}$\\
    LoRA LR & ablated & ablated\\
    AdamW $\beta_1$ & 0.9 & 0.9\\
    AdamW $\beta_2$ & 0.999 & 0.999\\
    Gradient Clip Norm & 1.0 & 1.0\\
    \cmidrule(lr) {1-1} \cmidrule(lr){2-3}
    Weight Decay & 0 & 0\\
    KL-regularization $\beta$ & $10^{-3}$ & $10^{-3}$\\
    Entropy regularization $\alpha$ & 0 & 0\\
    \cmidrule(lr) {1-1} \cmidrule(lr){2-3}
    Poisson Bootstrapping & True & False\\
    Ensemble Size & 16 & 1\\
    LoRA Rank & 128 & 128\\
    \cmidrule(lr) {1-1} \cmidrule(lr){2-3}
    Training Temperature & 1.0 & 1.0\\
    Validation Temperature & 0.6 & 0.6\\
    Max Output Tokens & 1536 & 1536\\
    \bottomrule
\end{tabular}
\label{tab:sweep_hyperparameters_rl_livecodebench}

\end{table}

\subsection{Why extended reasoning mode is disabled}
We disable reasoning mode in the considered LLMs for three critical reasons:
\begin{enumerate}
    \item \textbf{Credit assignment:} The credit assignment of \grpo{} on long sequences is highly sparse. As a result, disabling extended reasoning mode leads to more stable optimization trajectories for both \grpo{} and \name{} in black-box optimization. 
    \item \textbf{RL learns CoT:} In model post-training (RLVR), chain-of-thought is typically learnt via reinforcement learning. By disabling reasoning mode, the RLVR experiments become more meaningful and follow standard practice in the literature. Note that disabling thinking mode does not prevent the model from producing step-by-step reasoning in its response, but rather suppresses the extended pre-response thinking tokens. The RLVR setup therefore still allows CoT-style reasoning to emerge from RL pressure.\looseness-1
    \item \textbf{Compute cost:} Our total compute investment amounts to 145,000 H200 GPU hours. Enabling reasoning mode adds another order of magnitude to this cost. \name{} aims for computational efficiency; forcing CoT expands the compute gap against already expensive baselines like \es{}. Indeed, enabling reasoning mode for \es{} in preliminary experiments did not result in improved best-seen reward, resulted in frequent out-of-memory crashes, and reduced generation speed by at least one order of magnitude.
\end{enumerate}

\newpage
\section{Proofs}\label{sec:proofs}

\subsection{Proof sketches}
\begin{proof}[Proof sketch of Proposition~\ref{prop:poets_gradients}]
The proof begins by expanding the reward difference $r(a-\rho) - r_\pi(a-\rho)$ and showing that the intractable normalization constant $\log Z$ naturally cancels out, allowing the difference to be rewritten strictly in terms of the soft rewards $\tilde{r}_\pi(a) - \tilde{r}_\pi(\rho)$. By differentiating the squared loss with respect to the model parameters $\theta$, the chain rule produces the scaled gradient of the log-policy offset by the expected baseline. Finally, taking the outer expectation over the (empirical) baseline distribution $\rho$ eliminates the subtracted baseline term entirely because the average advantage evaluates to zero, perfectly recovering the final average gradient form.
\end{proof}

\begin{proof}[Proof sketch of Theorem~\ref{thm:cumulative_regret_bound}]
The proof heavily leverages the core Thompson Sampling identity, which establishes that, conditioned on the history, the expected objective value of the optimal policy under the true reward is equal to the expected objective value of the sampled policy under the sampled reward. Substituting this identity into the expected single-step soft regret, the KL-divergence and entropy regularization terms perfectly cancel each other out because they depend only on the policy and are independent of the reward functions. This cancellation reduces the expected regularized regret purely to the expected reward estimation error at the evaluation points, which is then bounded by $\mathcal O (\sqrt{T\gamma_T})$ using standard Gaussian process confidence bounds from prior works.
\end{proof}

\begin{proof}[Proof sketch of Proposition~\ref{prop:calibration_correlation_bound}.] 
Under the perfect calibration assumption $E^2 = V\cdot Z^2$ for actual error $E$, ensemble variance $V$, and standard normal noise $Z$. The expected squared noise is $1$, which conveniently reduces the covariance $\text{Cov}(V, E^2)$ strictly to $\text{Var}(V)$. However, the variance of the actual squared error, $\text{Var}(E^2)$, depends on the fourth moment of the Gaussian noise (which is $3$). This causes $\text{Var}(E^2)$ to expand to $3\text{Var}(V) + 2\mathbb{E}[V]^2$. Substituting these terms into the Pearson correlation definition results in the desired bound on the denominator: $\sqrt{3 + 2\mathbb{E}[V]^2/\text{Var}(V)} \geq \sqrt{3}$.
\end{proof}

\subsection{Formal proofs}

\begin{prop_}[Restatement of Proposition~\ref{prop:poets_gradients}]
    Let $\pi$ be parameterized by $\theta$ and define the KL-regularized reward $\tilde{r}_\pi(a) := r(a) + \beta \cdot (\log \pi_{ref}(a) - \log \pi(a))$. Then the gradients of the \name{} loss $L_{poets}^\rho$ are
    \begin{align}
        \nabla_\theta L_{poets}^\rho(a, \pi) & = -2\beta \cdot   \tilde{r}_\pi(a-\rho) (\nabla_\theta \log \pi(a) - \mathbb E_{a^\prime \sim \rho}[\nabla_\theta \log \pi(a^\prime)])\\
        \mathbb E_{a \sim \rho} [\nabla_\theta L_{poets}^\rho(a, \pi)] & = -2\beta \cdot \mathbb E_{a \sim \rho} \big[\tilde r_\pi(a-\rho) \nabla_\theta \log \pi(a)\big].
    \end{align}
\end{prop_}

\begin{proof}[Proof of Proposition~\ref{prop:poets_gradients}]
We start off by expanding the definition of $r(a-\rho)$ and $r_\pi(a-\rho)$ and canceling $\log Z$:
    \begin{align*}
        r(a-\rho) - r_\pi(a-\rho) & = r(a) - r_\pi(a) - (r(\rho) - r_\pi(\rho))\\ 
        & = r(a) - (\beta \log \pi(a) - \beta \log \pi_{ref}(a) + \beta \log Z)\\
        & \quad - \mathbb E_{a^\prime \sim \rho}[r(a^\prime) - (\beta \log \pi(a^\prime) - \beta \log \pi_{ref}(a^\prime) + \beta \log Z)]\\
        & = \tilde r_\pi(a) - \tilde r_\pi(\rho) = \tilde{r}_\pi(a-\rho).
    \end{align*}
    Using this identity, we next calculate the derivative of $L^{\rho}_{poets}(a, \pi)$ step by step:
    \begin{align*}
        \nabla_\theta L_{poets}^\rho(a, \pi) & = \nabla_\theta \big(r(a-\rho)-r_\pi(a-\rho)\big)^2 = -2(r(a-\rho)-r_\pi(a-\rho)) \nabla_\theta r_\pi(a-\rho)\\
        & = -2\tilde{r}_\pi(a-\rho) \nabla_\theta r_\pi(a-\rho)\\
        & = -2\beta \tilde{r}_\pi(a-\rho) (\nabla_\theta \log \pi(a) - \mathbb E_{a^\prime \sim \rho}[\nabla_\theta \log \pi(a^\prime)])
    \end{align*}
    For the second result, we take the outer expectation over $\rho$, i.e., the reward difference is taken with respect to the average reward on the entire dataset (or batch) $\{a_i\}$, formally described by the empirical distribution $\rho(a) = \tfrac{1}{G} \sum_{i=1}^G \delta_{a=a_i}$. The result then follows from $\mathbb E_{a \sim \rho}[\tilde{r}_\pi(a-\rho) \mathbb E_{a^\prime \sim \rho}[\nabla_\theta \log \pi(a^\prime)]] = \mathbb E_{a \sim \rho}[\tilde{r}_\pi(a-\rho)] \cdot \mathbb E_{a^\prime \sim \rho}[\nabla_\theta \log \pi(a^\prime)] = 0 \cdot \mathbb E_{a^\prime \sim \rho}[\nabla_\theta \log \pi(a^\prime)] = 0$
\end{proof}

\begin{thm_}[Detailed version of Theorem~\ref{thm:cumulative_regret_bound}]
    Suppose a Gaussian prior over the true reward vector $r$ with standard deviation $\leq 1$. Then, KL-regularized Thompson sampling (observing $r$ at $a_t \sim \pi_t$ for $\pi_t := \arg\max_\pi J(\pi, \tilde r_t)$ where $\tilde r_t \sim \mathbb P[\, r \mid \mathcal H_{t-1}]$ given observation history $\mathcal H_{t-1}$) has regret
    \begin{equation*}
        \textstyle \sum_{t=1}^T \mathbb E_{r,\pi_t} [\max_{\pi}J(\pi, r) - J(\pi_t, r)] \leq \zeta \sqrt{C_\eta} \cdot \sqrt{T \gamma_T}
        \in {\mathcal O}(\sqrt{T \gamma_T}).
    \end{equation*}
    $\zeta := 1 + \sqrt{2 \log (2 \cdot |\mathcal A |) + 2}$, $C_\eta:= {2}/{\ln(1+\eta^{-2})}$ where $\eta$ upper bounds the standard deviation of independent additive Gaussian observation noise, and $\gamma_T$ is the maximal information gain~\citep{srinivas2009gaussian} of the Gaussian prior over $r$. If $r$ is linear in $d$ features, then $\gamma_T \in \mathcal O(d \log T)$.
\end{thm_}

\begin{proof}[Proof of Theorem~\ref{thm:cumulative_regret_bound}]
    Let $r$ be the true reward, let $\mathcal{H}_{t-1}$ denote the observation history up to time $t$, and let $\tilde r_t \sim \mathbb P[\, r \mid \mathcal H_{t-1}]$. Define the optimal policy $\pi^* := \arg\max_\pi J(\pi, r)$. By definition, KL-regularized Thompson sampling draws from $\pi_t := \arg\max_\pi J(\pi, \tilde r_t)$. Now, note that $\pi^*, r \overset{d}{=} \pi_t, \tilde r_t$ given $\mathcal H_{t-1}$. As a result, the (soft) reward conditioned on $\mathcal H_{t-1}$ is also identically distributed:
    \begin{equation*}
        \mathbb E [J(\pi^*, r) \mid \mathcal{H}_{t-1}] = \mathbb E [J(\pi_t, \tilde{r}_t) \mid \mathcal{H}_{t-1}].
    \end{equation*}
    Defining $\Delta_t := J(\pi^*, r) - J(\pi_t, r)$, we substitute this identity into the expected single-step regret:
    \begin{align*}
        \mathbb E [\Delta_t \mid \mathcal{H}_{t-1}] = \mathbb E [J(\pi^*, r) - J(\pi_t, r) \mid \mathcal{H}_{t-1}] = \mathbb E [J(\pi_t, \tilde{r}_t) - J(\pi_t, r) \mid \mathcal{H}_{t-1}].
    \end{align*}
    We expand the objective function $J$ for the same policy $\pi_t$ under both reward functions $\tilde{r}_t$ and $r$:
    \begin{align*}
        \mathbb E [\Delta_t \mid \mathcal{H}_{t-1}] & = \textstyle \mathbb E \Big[ \big( \sum_a \tilde{r}_t(a) \pi_t(a)  - \beta D_{KL}(\pi_t || \pi_{ref}) \big) \\
        & \textstyle \quad - \big( \mathbb \sum_a r(a) \pi_t(a) - \beta D_{KL}(\pi_t || \pi_{ref}) \big) \mid \mathcal{H}_{t-1} \Big].
    \end{align*}
    Note that here $\pi_t$ is itself a random quantity that depends on $\tilde r_t$. Because the KL-divergence and entropy terms depend solely on the policy $\pi_t$ and are independent of the reward, they cancel out:
    \begin{equation*}
        \textstyle \mathbb E [\Delta_t \mid \mathcal{H}_{t-1}] = \mathbb E \big[ \sum_a (\tilde{r}_t(a) - r(a)) \cdot \pi_t(a) \mid \mathcal{H}_{t-1} \big].
    \end{equation*}
    The expected regularized regret simplifies entirely to the expected reward estimation error at evaluation points. We next bound this term using standard confidence bounds over $T$ steps, in particular via the maximal information gain~\citep{srinivas2009gaussian}. More precisely, according to Lemma~\ref{lem:uniform_bound}, it holds that
    \begin{equation*}
        \textstyle \mathbb E [\Delta_t \mid \mathcal{H}_{t-1}] \leq \mathbb E \big[ \sum_a |\tilde{r}_t(a) - r(a)| \cdot \pi_t(a) \mid \mathcal{H}_{t-1} \big] \leq \zeta \sqrt{\mathbb E [\sum_a \sigma_t^2(a) \cdot \pi_t(a) \mid \mathcal{H}_{t-1} \big]}.
    \end{equation*}
    Here, the last step used that under Gaussian priors $r, \tilde r_t$ are i.i.d. element-wise $\sigma_t$-subgaussian given observation history $\mathcal H_{t-1}$, where $\sigma_t(a)$ is the posterior standard deviation for the reward of action $a$. Next, we apply Jensen's inequality and Cauchy-Schwarz across the sequence dimension to obtain
    \begin{align*}
        & \textstyle \sum_{t=1}^T \mathbb E_{r,\pi_t} [\max_{\pi}J(\pi, r) - J(\pi_t, r)] \textstyle = \sum_{t=1}^T \mathbb E[\Delta_t] = \sum_{t=1}^T \mathbb E[\mathbb E [\Delta_t \mid \mathcal{H}_{t-1}]]\\
   \leq & \textstyle \sum_{t=1}^T \mathbb E \big[\zeta \sqrt{\mathbb E [\sum_a \sigma_t^2(a) \cdot \pi_t(a) \mid \mathcal{H}_{t-1}]}\big] \leq \zeta \sum_{t=1}^T 1 \cdot \sqrt{\mathbb E [\sum_a \sigma_t^2(a) \cdot \pi_t(a)]}\\
   \leq & \textstyle \zeta \sqrt{T \cdot \sum_{t=1}^T \mathbb E[\sum_a \sigma_t^2(a) \cdot \pi_t(a)]} = \zeta \sqrt{T \cdot \mathbb E [ \sum_{t=1}^T \sigma_t^2(a_t)]} \text{ where } a_t \sim \pi_t.
    \end{align*}
    To finish the proof, we apply Lemma~\ref{lem:information_gain_for_gaussians}, which results in the familiar ${\mathcal O}(\sqrt{T \gamma_T})$ scaling. Note that by rescaling the reward, the technical condition $\sigma_1(a) \leq 1 \ \forall a$ holds without loss of generality and can be absorbed as a constant in $\mathcal O$ notation. For bounds on $\gamma_T$, we refer to~\cite{srinivas2009gaussian}.
\end{proof}

\begin{prop_}[Restatement of Proposition~\ref{prop:calibration_correlation_bound}]
Let $V > 0$ be a random variable representing the predicted variance, and let $E$ be the signed error, such that $E^2$ is the squared error. Assume the model is perfectly calibrated under Gaussian assumptions, i.e., $E^2 = V \cdot Z^2$ for $Z \sim \mathcal{N}(0, 1)$ independent of $V$. Then, the Pearson correlation $\rho(V, E^2)$ is strictly bounded from above:\looseness-1
$$ \rho(V, E^2) \leq \frac{1}{\sqrt{3}} \approx 0.577 $$
\end{prop_}

\begin{proof}
The Pearson correlation is defined as:
$$ \rho(V, E^2) = \frac{\text{Cov}(V, E^2)}{\sqrt{\text{Var}(V)\text{Var}(E^2)}} $$
Because $Z \sim \mathcal{N}(0, 1)$, its relevant moments are $\mathbb{E}[Z^2] = 1$ and $\mathbb{E}[Z^4] = 3$. Since $V$ and $Z$ are independent, we first compute the expectations required for the covariance:
$$ \mathbb{E}[E^2] = \mathbb{E}[V Z^2] = \mathbb{E}[V]\mathbb{E}[Z^2] = \mathbb{E}[V] $$
$$ \mathbb{E}[V E^2] = \mathbb{E}[V^2 Z^2] = \mathbb{E}[V^2]\mathbb{E}[Z^2] = \mathbb{E}[V^2] $$
$$ \text{Cov}(V, E^2) = \mathbb{E}[V E^2] - \mathbb{E}[V]\mathbb{E}[E^2] = \mathbb{E}[V^2] - \mathbb{E}[V]^2 = \text{Var}(V) $$
Next, we determine the variance of $E^2$ by looking at the expectation of $E^4$:
$$ \mathbb{E}[E^4] = \mathbb{E}[V^2 Z^4] = \mathbb{E}[V^2]\mathbb{E}[Z^4] = 3\mathbb{E}[V^2] $$
Using the standard identity $\mathbb{E}[V^2] = \text{Var}(V) + \mathbb{E}[V]^2$, we expand $\text{Var}(E^2)$:
$$ \text{Var}(E^2) = \mathbb{E}[E^4] - \mathbb{E}[E^2]^2 = 3(\text{Var}(V) + \mathbb{E}[V]^2) - \mathbb{E}[V]^2 = 3\text{Var}(V) + 2\mathbb{E}[V]^2 $$
Substituting these values back into the Pearson correlation formula yields the desired statement:
$$ \rho(V, E^2) = \frac{\text{Var}(V)}{\sqrt{\text{Var}(V)(3\text{Var}(V) + 2\mathbb{E}[V]^2)}} = \frac{1}{\sqrt{3 + 2\frac{\mathbb{E}[V]^2}{\text{Var}(V)}}} \leq \frac{1}{\sqrt{3}}.$$
\end{proof}

\subsection{Lemmas}

\begin{lem}[]\label{lem:uniform_bound}
    Let $r, r^\prime$ be i.i.d., element-wise $\sigma_x$-subgaussian random vectors, and let $x \sim p_r \in \{1, \ldots, |\mathcal A |\}$ for any distribution $p_r$ parameterized by $r$. Then 
    \begin{align*}
        \mathbb E[|r_x-\mu_x|] & \leq \sqrt{\mathbb E[\sigma_{x}^2] \cdot (2 \log (2|\mathcal A |)+2)} \text{ and }\\
        \mathbb E[|r_x-r^\prime_x|] & \leq \zeta \cdot \sqrt{\mathbb E[\sigma_{x}^2]}.
    \end{align*}
    for $\zeta := 1 + \sqrt{2 \log(2\cdot|\mathcal A |) + 2}$.
\end{lem}
\begin{proof}
    We generalize the proof given by \cite{menet2026codingagentsenvironmentinteraction} from deterministic $x = f(r)$ to distributions $x \sim p_r$. Define $Z_z = (r_z - \mu_z) / \sigma_z$. Since $r_z - \mu_z$ is $\sigma_z$-subgaussian, $Z_z$ is $1$-subgaussian. Cauchy-Schwarz then gives, $\mathbb E[|r_x - \mu_x|] = \mathbb E [\sigma_{x} | Z_x|] \leq \sqrt{\mathbb E[\sigma_{x}^2] \mathbb E[Z_x^2]}$. Now, $Z_x^2 \leq 
    \max_{1 \leq j \leq |\mathcal A |} Z_j^2$, so the tail integral formula for expectations with union bound gives $\textstyle \mathbb E[Z_x^2] \leq \mathbb E[\max_{1 \leq j \leq |\mathcal A |} Z_j^2] = \int_0^\infty\! \mathbb P[\max_{1\leq j \leq |\mathcal A |} Z_j^2 \geq t] dt \leq \int_0^\infty \!\min(1, 2|\mathcal A | e^{-t/2}) dt =  2 \log (2|\mathcal A |)+2$. The first result then follows immediately. For the second result, apply the triangle inequality $\mathbb E[|r_x-r^\prime_x|] \leq \mathbb E[|r_x-\mu_x|] + \mathbb E[|r_x^\prime-\mu_x|]$, bound the first term as above, and the second by $\sqrt{\mathbb E[\sigma_{x}^2]}$, using that $r^\prime \perp x$ and thus $\mathbb E[(Z_x^\prime)^2] \leq 1$.
\end{proof}

\begin{lem}\label{lem:information_gain_for_gaussians}
    Let $r \sim \mathcal N(\mu, \Sigma)$ be a multivariate Gaussian with uniformly bounded standard deviations $\sigma_1(x) \leq 1\ \forall x$ that is consecutively evaluated at locations $(x_t)_{t=1}^T$ with observations $y_t = r({x_t}) + \varepsilon_t$. The noise is i.i.d. $\varepsilon_t \sim \mathcal N(0, \eta^2_t)$ with $\eta_{t} \leq \eta\ \forall t$. Then, the maximum information gain $\gamma_T$ upper bounds the aggregated predictive variances $\sigma_t^2(x) := \mathrm{Var}[r(x) \mid y_1, \ldots, y_{t-1}]$ at the evaluation locations, i.e.,
    \begin{equation*}
        \textstyle \gamma_T:= \max_{x_1, \ldots, x_T} I(y_{1:T}; r) \geq I(y_{1:T}; r) \geq \underbrace{\tfrac{\ln(1+\eta^{-2})}{2}}_{=: 1/C_{\eta}} \sum_{t=1}^T \sigma_t^2(x_t).
    \end{equation*}
\end{lem}

\begin{proof}
    We repeat the proof given by~\cite{menet2025lite}, which generalizes the proof of \citeauthor{srinivas2009gaussian} \citeyearpar{srinivas2009gaussian} from homoscedastic to heteroscedastic noise. First, note that the expression on the right hand side is well-defined, because $\sigma_t(x_t)$ only depends on the observation locations $x_1, \ldots, x_{t-1}$, but not on the observed value. Moreover, since Gaussian conditioning cannot increase the variance, $\sigma_t(x_t) \leq \sigma_1(x_t) \leq 1\ \forall x$. Next, note that $y_{1:T} \mid r$ is a multivariate normal with independent components of variance $\eta_{1}^2, \ldots, \eta_{T}^2$, thus
    \begin{align*}
        \textstyle I(y_{1:T}; r) = H[y_{1:T}] - H[y_{1:T} | r] = H[y_{1:T}] - \tfrac{1}{2} \sum_{t=1}^T \ln (2 \pi e \eta^2_t).
    \end{align*}
    Furthermore, one may decompose 
    \begin{align*}
        \textstyle H[y_{1:T}] = H[y_{1:{T\text{-}1}}] + H[y_{T}|y_{1:{T\text{-}1}}] = H[y_{1:{T\text{-}1}}] + \tfrac{1}{2}\ln(2\pi e (\eta^2_{T} + \sigma_{T}^2(x_T))),
    \end{align*}
    using that $y_{T} | y_{1:{T\text{-}1}}$ is Gaussian with variance $\eta^2_{T} + \sigma_{T}^2(x_T)$. Recursively expanding then results in 
    $$\textstyle I(y_{1:T}; r) = \frac{1}{2} \sum_{t=1}^T \ln(1+\eta^{\text{-}2}_{t} \sigma_{t}^2(x_t)).$$
    Finally, by assumption $\sigma_{t}(x_t) \in [0,1]$, allowing to lower bound each summand
    \begin{align*}
    \sigma_{t}^2(x_t) \leq  \tfrac{1}{2}\ln(1+\eta^{\text{-}2}_{t}\sigma_{t}^2(x_t))\underbrace{\tfrac{2}{\ln(1+\eta^{\text{-}2}_{t})}}_{=:C_{\eta_{{t}}}}.
    \end{align*}
    The lower bound follows from $g(\sigma_{t}^2(x_t)) := \ln(1+\eta^{\text{-}2}_t \sigma_{t}(x_t)^2) - \sigma_{t}^2(x_t) \ln(1+\eta^{\text{-}2}_{t}) \geq 0\ \forall \sigma_{t}(x_t) \in [0,1]$, which in turn follows from $g(0) = 0$, $g(1) = 0$, and concavity of $g$.
\end{proof}

\newpage
\section{Broader impact, assets, and licenses}

\subsection{Broader impact}\label{sec:broader_impact}
This work proposes \name{}, a general-purpose framework for uncertainty-aware exploration in reinforcement learning and black-box optimization. By improving sample efficiency and mitigating premature convergence, \name{} may positively impact scientific discovery settings where evaluations are costly, such as protein design, therapeutic molecular design, quantum circuit synthesis, and other combinatorial optimization problems. In LLM post-training, the method can improve robustness by maintaining exploration in high-uncertainty regimes and enabling effective use of experience replay.

As a foundational algorithmic contribution, \name{} does not introduce new datasets, release trained models, or directly target downstream applications. Nevertheless, like other methods that improve the efficiency of reward-driven optimization, it could indirectly accelerate the development of more capable models, amplifying both beneficial and harmful uses depending on deployment. \name{} does not address reward misspecification, fairness, or alignment; these concerns remain the responsibility of system designers through careful reward construction, evaluation, and deployment safeguards.

\subsection{Assets and licenses}\label{sec:licenses}
In this work, we utilize several open-source models, datasets, and software libraries. We acknowledge the creators of these assets and respect their licensing terms:

\begin{itemize}
    \item \textbf{Models:} We utilize the \textbf{Qwen3-8B} and \textbf{Qwen3-Embedding-0.6B} models \citep{yang2025qwen3technicalreport, qwen3embedding}, which are released by Alibaba Cloud under the Apache 2.0 License. We also use \textbf{ProtGPT2} \citep{ferruz2022protgpt2}, built on the GPT-2 architecture~\citep{radford2019language}, which is also made available under the Apache 2.0 License. Finally, we report on \textbf{GP-MoLFormer-Uniq}~\citep{ross2025gp}, released by IBM Research under Apache 2.0.
    
    \item \textbf{Datasets:} We evaluate our reinforcement learning approach using problems from the American Invitational Mathematics Examination (\textbf{AIME}) and LiveCodeBench~\citep{jain2025livecodebench}. The AIME problems are the copyrighted property of the Mathematical Association of America (MAA) and are used here strictly for non-commercial academic research purposes. The compiled LiveCodeBench dataset is released under the Creative Commons license (CC BY-SA 4.0), while the underlying problem text remains the property of the respective competitive programming platforms.
    \item \textbf{Software \& Frameworks:} Our framework is implemented in \textbf{PyTorch} \citep{paszke2019pytorchimperativestylehighperformance}, licensed under the Modified BSD License. We use the \textbf{Hugging Face Transformers} library \citep{wolf2020huggingfacestransformersstateoftheartnatural} and \textbf{Qiskit} \citep{javadiabhari2024quantumcomputingqiskit}, both of which are licensed under the Apache 2.0 License. For the molecular tasks, we calculate descriptors using \textbf{RDKit} \citep{rdkit} (licensed under the BSD 3-Clause License) and query target binding oracles using the \textbf{PyTDC} library \citep{huang2021therapeuticsdatacommonsmachine} (licensed under the MIT License).
\end{itemize}

\newpage
\section*{NeurIPS Paper Checklist}

\begin{enumerate}

\item {\bf Claims}
    \item[] Question: Do the main claims made in the abstract and introduction accurately reflect the paper's contributions and scope?
    \item[] Answer: \answerYes{} 
    \item[] Justification: The abstract and introduction clearly outline the claims about introducing the \name{} framework, its theoretical cumulative regret bounds, and its empirical performance. These claims are directly supported by the proofs in Appendix~\ref{sec:proofs} and the results in the experimental section.

\item {\bf Limitations}
    \item[] Question: Does the paper discuss the limitations of the work performed by the authors?
    \item[] Answer: \answerYes{} 
    \item[] Justification: Section~\ref{sec:conclusion_and_limitations} discusses the limitations of \name{} in the context of future work. The computational efficiency of our method is benchmarked in detail in Section~\ref{sec:computational_overhead}.

\item {\bf Theory assumptions and proofs}
    \item[] Question: For each theoretical result, does the paper provide the full set of assumptions and a complete (and correct) proof?
    \item[] Answer: \answerYes{} 
    \item[] Justification: The paper presents its theoretical results clearly in Section~\ref{sec:method} and Section~\ref{sec:theory} of the main text. In Appendix~\ref{sec:proofs}, both proof sketches and formal mathematical proofs are provided.

    \item {\bf Experimental result reproducibility}
    \item[] Question: Does the paper fully disclose all the information needed to reproduce the main experimental results of the paper to the extent that it affects the main claims and/or conclusions of the paper (regardless of whether the code and data are provided or not)?
    \item[] Answer: \answerYes{} 
    \item[] Justification: Appendix~\ref{sec:experimental_details} entails a comprehensive "Experimental Details" section, which outlines exact hyperparameter settings, prompt templates, and evaluation environments. Additionally, the entire codebase is available in the supplementary materials for reproducibility.

\item {\bf Open access to data and code}
    \item[] Question: Does the paper provide open access to the data and code, with sufficient instructions to faithfully reproduce the main experimental results, as described in supplemental material?
    \item[] Answer: \answerYes{} 
    \item[] Justification: The entire code base, complete with experimental configuration files and environment files specifying required dependencies, was provided in the supplementary materials during the reviewing process and is now publicly available at \url{https://github.com/IBM/thompson-sampling-via-fine-tuning-of-llms}.

\item {\bf Experimental setting/details}
    \item[] Question: Does the paper specify all the training and test details (e.g., data splits, hyperparameters, how they were chosen, type of optimizer) necessary to understand the results?
    \item[] Answer: \answerYes{} 
    \item[] Justification: Appendix~\ref{sec:experimental_details} meticulously specifies all necessary experimental settings, including the use of the AdamW optimizer, exact learning rates for shared backbones and LoRA branches, gradient clipping parameters, and prompt instructions.

\item {\bf Experiment statistical significance}
    \item[] Question: Does the paper report error bars suitably and correctly defined or other appropriate information about the statistical significance of the experiments?
    \item[] Answer: \answerYes{} 
    \item[] Justification: To ensure statistical significance, the experiments report the mean and (its) standard error for all experimental metrics across 25 random seeds.

\item {\bf Experiments compute resources}
    \item[] Question: For each experiment, does the paper provide sufficient information on the computer resources (type of compute workers, memory, time of execution) needed to reproduce the experiments?
    \item[] Answer: \answerYes{} 
    \item[] Justification: The paper explicitly states that all experiments were conducted on a single NVIDIA H200 GPU with 141GB of memory. Moreover, Appendix~\ref{sec:experimental_details} states the compute hours for each of the reported experiments as well as a total that includes preliminary experiments and hyper-parameter tuning.
    
\item {\bf Code of ethics}
    \item[] Question: Does the research conducted in the paper conform, in every respect, with the NeurIPS Code of Ethics \url{https://neurips.cc/public/EthicsGuidelines}?
    \item[] Answer: \answerYes{} 
    \item[] Justification: The research revolves around the development of a theoretical and algorithmic exploration framework for LLMs and does not contain methods, applications, or subject matter that inherently violate the NeurIPS Code of Ethics.

\item {\bf Broader impacts}
    \item[] Question: Does the paper discuss both potential positive societal impacts and negative societal impacts of the work performed?
    \item[] Answer: \answerYes{} 
    \item[] Justification: A dedicated broader impacts section is provided in the Appendix (Section~\ref{sec:broader_impact}).
    
\item {\bf Safeguards}
    \item[] Question: Does the paper describe safeguards that have been put in place for responsible release of data or models that have a high risk for misuse (e.g., pre-trained language models, image generators, or scraped datasets)?
    \item[] Answer: \answerNA{} 
    \item[] Justification: The contribution of this work is the \name{} algorithmic framework; it does not involve the release of new language models, image generators, or scraped datasets

\item {\bf Licenses for existing assets}
    \item[] Question: Are the creators or original owners of assets (e.g., code, data, models), used in the paper, properly credited and are the license and terms of use explicitly mentioned and properly respected?
    \item[] Answer: \answerYes{} 
    \item[] Justification: All owners of code, data, and models are properly credited and the respective licenses are respected and explicitly mentioned within a dedicated Section~\ref{sec:licenses} in the Appendix.

\item {\bf New assets}
    \item[] Question: Are new assets introduced in the paper well documented and is the documentation provided alongside the assets?
    \item[] Answer: \answerYes{} 
    \item[] Justification: \name{} is a novel framework for sample-efficient reinforcement learning. The full codebase is provided in the supplementary materials, including proper documentation.

\item {\bf Crowdsourcing and research with human subjects}
    \item[] Question: For crowdsourcing experiments and research with human subjects, does the paper include the full text of instructions given to participants and screenshots, if applicable, as well as details about compensation (if any)? 
    \item[] Answer: \answerNA{} 
    \item[] Justification: The experimental setup relies entirely on automated black-box optimization benchmarks (such as FAQ refinement, Protein Search, and Quantum Circuit Design, Therapeutics Data Commons) and existing mathematical or coding datasets. Therefore, the research does not involve crowdsourcing or human subjects.

\item {\bf Institutional review board (IRB) approvals or equivalent for research with human subjects}
    \item[] Question: Does the paper describe potential risks incurred by study participants, whether such risks were disclosed to the subjects, and whether Institutional Review Board (IRB) approvals (or an equivalent approval/review based on the requirements of your country or institution) were obtained?
    \item[] Answer: \answerNA{} 
    \item[] Justification: Because the paper's methodology does not involve research with human subjects, IRB approvals or equivalent reviews are not applicable.

\item {\bf Declaration of LLM usage}
    \item[] Question: Does the paper describe the usage of LLMs if it is an important, original, or non-standard component of the core methods in this research? Note that if the LLM is used only for writing, editing, or formatting purposes and does \emph{not} impact the core methodology, scientific rigor, or originality of the research, declaration is not required.
    \item[] Answer: \answerYes{} 
    \item[] Justification: LLMs are the central focus of the optimization method proposed in this research. Thus, their usage is comprehensively documented, including system instructions, generation prompts, and fine-tuning details throughout the main text and Appendix~\ref{sec:experimental_details}.

\end{enumerate}

\end{document}